\documentclass[11pt]{article}

\usepackage{fullpage,times,url,natbib}

\usepackage{amsthm,amsfonts,amsmath,amssymb,epsfig,color,float,graphicx,verbatim}
\usepackage{algorithm,algorithmic}

\usepackage{enumitem}

\usepackage{hyperref}
\hypersetup{
	colorlinks   = true, 
	urlcolor     = blue, 
	linkcolor    = blue, 
	citecolor   = blue 
}

\newtheorem{theorem}{Theorem}
\newtheorem{proposition}{Proposition}
\newtheorem{lemma}{Lemma}
\newtheorem{corollary}{Corollary}
\newtheorem{definition}{Definition}
\newtheorem{remark}{Remark}
\newtheorem{assumption}{Assumption}

\newcommand{\reals}{\mathbb{R}}
\newcommand{\E}{\mathbb{E}}

\newcommand{\bx}{\mathbf{x}}
\newcommand{\bw}{\mathbf{w}}

\newcommand{\bu}{\mathbf{u}}
\newcommand{\bv}{\mathbf{v}}
\newcommand{\bz}{\mathbf{z}}

\newcommand{\by}{\mathbf{y}}

\newcommand{\bs}{\mathbf{s}}
\newcommand{\br}{\mathbf{r}}

\newcommand{\balpha}{\boldsymbol{\alpha}}
\newcommand{\bbeta}{\boldsymbol{\beta}}

\newcommand{\Ocal}{\mathcal{O}}

\newcommand{\Dcal}{\mathcal{D}}

\newcommand{\Vcal}{\mathcal{V}}

\newcommand{\Ucal}{\mathcal{U}}

\newcommand{\norm}[1]{\|#1\|}

\newtheorem{example}{Example}

\newcommand{\secref}[1]{Sec.~\ref{#1}}

\newcommand{\figref}[1]{Fig.~\ref{#1}}
\renewcommand{\eqref}[1]{Eq.~(\ref{#1})}
\newcommand{\lemref}[1]{Lemma~\ref{#1}}
\newcommand{\corollaryref}[1]{Corollary~\ref{#1}}
\newcommand{\thmref}[1]{Thm.~\ref{#1}}
\newcommand{\propref}[1]{Proposition~\ref{#1}}

\title{The Implicit Bias of Benign Overfitting}
\author{Ohad Shamir \\ Weizmann Institute of Science \\ \texttt{ohad.shamir@weizmann.ac.il}}
\date{}

\begin{document}

\maketitle

\begin{abstract}%
	The phenomenon of benign overfitting, where a predictor perfectly fits noisy training data while attaining near-optimal expected loss, has received much attention in recent years, but still remains not fully understood beyond well-specified linear regression setups. In this paper, we provide several new results on when one can or cannot expect benign overfitting to occur, for both regression and classification tasks. We consider a prototypical and rather generic data model for benign overfitting of linear predictors, where an arbitrary input distribution of some fixed dimension $k$ is concatenated with a high-dimensional distribution. For linear regression which is not necessarily well-specified, we show that the minimum-norm interpolating predictor (that standard training methods converge to) is biased towards an inconsistent solution in general, hence benign overfitting will generally \emph{not} occur. Moreover, we show how this can be extended beyond standard linear regression, by an argument proving how the existence of benign overfitting on some regression problems precludes its existence on other regression problems. We then turn to classification problems, and show that the situation there is much more favorable. Specifically, we prove that the max-margin predictor (to which standard training methods are known to converge in direction) is asymptotically biased towards minimizing a weighted \emph{squared hinge loss}. This allows us to reduce the question of benign overfitting in classification to the simpler question of whether this loss is a good surrogate for the misclassification error, and use it to show benign overfitting in some new settings. 
\end{abstract}

\section{Introduction}

The ability of learning algorithms to succeed despite overfitting is a curious phenomenon in statistical learning, which has received much interest in the past few years. A particular version of it, commonly denoted as ``benign overfitting'' \citep{bartlett2020benign}, refers to situations which combine the following: (1) The trained predictor interpolates the data, in the sense that it achieves perfect prediction accuracy on the training data; (2) No predictor in the relevant hypothesis class can achieve perfect accuracy w.r.t. the underlying data distribution; yet (3) The trained predictor has near-optimal accuracy w.r.t. the underlying data distribution. The combination of $(1)$ and $(2)$ implies that the predictor overfits (in the sense that the training error is significantly smaller than the test error), and $(3)$ implies that this overfitting is ``benign'' (in the simple sense that the predictor still performs well). This phenomenon is intriguing, because some version of it appears to occur frequently in large-scale learning problems, yet it cannot be easily explained using standard learning theoretic tools such as uniform convergence (which requires the performance on the training data and the underlying distribution to be similar). This has led to a flurry of papers in the past few years, attempting to understand why and when benign overfitting occurs, and whether uniform convergence can or cannot explain its occurence (see discussion of related work below). 

So far, most of the theoretical work on benign overfitting has focused on linear (or kernel) regression problems using the square loss, with some works extending this to classification problems. The relatively most well-understood situation is plain linear regression in a well-specified setting, where we are training a linear predictor $\bx\mapsto \bx^\top \bw$, $\bw\in \reals^d$ with respect to the square loss, and when the outputs $y$ satisfy $y=\bx^\top \bw^*+\xi$ for some $\bw^*\in \reals^d$, where $\xi$ is zero-mean noise. In this setting, we know that benign overfitting occurs (roughly speaking) whenever the dimension of the input $\bx$ is sufficiently large, and the distribution has many directions of small (but non-zero) variance that the trained predictor can utilize to perfectly fit the training data, without significantly affecting the distribution of the predictions on new samples. A helpful feature of this setting is that there is a closed-form expression for the predictor returned by standard gradient-based  methods, when trained to convergence on the average square loss -- namely, the minimum-norm interpolating predictor. For classification, the situation is more complicated, because the predictor that standard gradient-based methods converge to with appropriate losses (namely, the max-margin predictor) does not have a closed-form expression in general. Thus, most recent works on classification focused on more specific setups (as discussed in the related work section below).

In this paper, we provide several new results on benign overfitting, which can be used to analyze when benign overfitting may -- or may not -- occur in settings beyond those studied so far in the literature. We focus on a prototypical data model for benign overfitting, where for some fixed integer $k$, the input distribution in $\reals^d$ is composed of an arbitrary distribution on the first $k$ coordinates, and a \emph{high-dimensional} distribution on the last $d-k$ coordinates (in the sense that i.i.d. samples from that distribution are approximately orthogonal). As we discuss later on, such a setting is essentially necessary for benign overfitting to occur, even for well-specified linear regression. In this setting, we study benign overfitting by considering the asymptotic consistency of the predictor learned by standard gradient-based methods, as both $d$ and the sample size $m$ diverge to infinity, and $d$ increases sufficiently faster than $m$. Our results can be informally summarized as follows:
\begin{itemize}[leftmargin=*]
	\item Beginning with linear regression, we show that once the distribution is not well-specified, the minimum-norm interpolating predictor (a.k.a. min-norm predictor) returned by standard training methods will generally \emph{not} be consistent, and hence benign overfitting will generally not occur. A bit more concretely, in our data model a consistent predictor $\hat{\bw}$ should be such that asymptotically, 
	\begin{equation}\label{intro:regress1}
	\hat{\bw}_{|k} \approx \left(\E\left[\bx\bx^\top\right]^{-1}\right)_{k}\E[y\bx]~,
	\end{equation}
	where $\bv_{|k}$ of a vector $\bv$ denote its first $k$ coordinates, and  $\left(\E\left[\bx\bx^\top\right]^{-1}\right)_{k}$ refers to the first $k$ rows of the inverse covariance matrix. In other words, we expect the first $k$ coordinates to converge to the first $k$ coordinates of the least-squares solution w.r.t. the underlying distribution. In contrast, we prove that the min-norm predictor asymptotically satisfies
	\begin{equation}\label{intro:regress2}
	\hat{\bw}_{|k} \approx \E\left[\frac{\bx_{|k}\bx_{|k}^\top}{\norm{\bx_{|d-k}}^2}\right]^{-1}\E\left[\frac{y\bx_{|k}}{\norm{\bx_{|d-k}}^2}\right]~,
	\end{equation}
	where $\bx_{|d-k}$ are the last $d-k$ coordinates of $\bx$. Clearly, the expressions in  \eqref{intro:regress1} and \eqref{intro:regress2} are generally not the same, except in some special cases (an important one, as we shall see, being a well-specified setting). Thus, we argue that one should not expect benign overfitting to generally occur in linear regression with the square loss. 
	\item We show that benign overfitting will generally not occur in some natural extensions of linear regression with the square loss, even in a well-specified setting. These include (1) regression with generalized linear models (or equivalently, with a single neuron predictor); and (2) regression with losses other than the square loss. To do so, we present a new observation that may be of independent interest. Roughly speaking, we argue that any interpolating predictor can be seen as returning the optimum of the average loss, but simultaneously, it is also the optimum of many other types of average loss objectives, which reflect rather different learning problems with different optimal predictors. The trained predictor does not ``know'' which of these learning problems it is actually solving, so benign overfitting (with the trained predictor achieving low expected loss) can only occur in some of them. As a result, benign overfitting is implicitly ``biased'' towards certain learning problems, and its occurence in one problem precludes its occurence in another. We note that this is somewhat reminiscent of the paper \citet{muthukumar2021classification}, which pointed out that interpolating predictors can be insensitive to the type of loss function used for training, but we take this in a rather different direction.
	\item Having discussed regression problems, we turn to binary classification problems (where our goal is to minimize misclassification error), and show that the situation there is rather different. Concretely, we consider linear predictors under the same data model as before, and the max-margin predictor $\hat{\bw}$ (to which gradient-based methods are known to asymptotically converge in direction). Perhaps surprisingly, we prove that $\hat{\bw}$ has a rather clean asymptotic characterization: Its first $k$ coordinates asymptotically minimize the expectation of a (weighted) \emph{squared hinge loss} on those coordinates,
	\[
	\hat{\bw}_{|k}~\approx~\arg\min_{\bv\in \reals^k}\E\left[\frac{[1-y\bx_{|k}^\top \bv]_+^2}{\norm{\bx_{|d-k}}^2}\right]~,
	\]
	(where $[z]_+:=\max\{0,z\}$), and the last $d-k$ coordinates of $\hat{\bw}$ are asymptotically immaterial. Thus, we get that the existence of benign overfitting in our model is reduced to a simpler question: Whether the data distribution is such that minimizing the expectation of this weighted squared hinge loss is a good surrogate for minimizing misclassification error. Although not true in the worst case, it is reasonable to assume that it will be true in many cases, because this loss still encourages the sign of $\bx_{|k}^\top\hat{\bw}_{k}$ to accord with the output $y$. This is in contrast to regression, in which our results suggest that benign overfitting is more brittle. Also, we note that unlike many previous works on benign overfitting in classification, our result does not require that the distribution is such that the min-norm and max-margin predictors coincide. Based on this result, we study more specifically the case of linearly separable distributions with label noise, and provide a few positive results: For example, for just about any choice of distribution on the first $k$ coordinates, we will have benign overfitting at least for some positive amount of label noise. Moreover, under some stronger assumptions on the input distribution (e.g., a mixture of symmetric distributions), we will have benign overfitting for any non-trivial level of label noise.
\end{itemize}

Overall, we hope that the results provided in this paper will allow us to understand the phenomenon of benign overfitting beyond the settings studied so far. 

The paper is structured as follows: After discussing related work below, we study linear regression with the square loss in \secref{sec:regression}, and other regression settings in \secref{sec:beyond}. We then turn to classification problems in \secref{sec:classification}, and conclude with a discussion in  \secref{sec:discussion}. The formal proofs of all our results appear in Appendix \ref{app:proofs}. 

\subsection{Related Work}

Papers such as \citep{Zhang2017} popularized the notion that modern learning systems (such as deep learning) tend to perfectly fit the training data, while still performing well on test data. The literature on the theory of this phenomenon is by now very large, and we will only discuss here the papers most relevant to our work (see for example \cite{belkin2021fit} for a more comprehensive survey). 

A line of works (e.g., \cite{belkin2018overfitting,belkin2018understand,belkin2018understand,belkin2019does,mei2019generalization,liang2020just,belkin2019reconciling}) showed that this phenomenon is not reserved to deep learning, and occurs already in linear and kernel learning. More recently, papers such as \citet{bartlett2020benign,hastie2019surprises,belkin2020two} studied conditions for benign overfitting in linear regression with the square loss in a well-specified setting. In particular, \citet{bartlett2020benign} considered general distributions, and showed how the occurrence of benign overfitting can be characterized in terms of the eigenvalues of the input covariance matrix, and how having many low-variance directions is in some sense necessary for benign overfitting to occur. \citet{koehler2021uniform,zhou2021optimistic} showed how these results can be recovered and extended from a uniform convergence perspective, again for well-specified linear regression. Other works which study benign overfitting and its relationship to classical learning theory include \citet{nagarajan2019uniform,negrea2020defense,yang2021exact,bartlett2021failures,bachmann2021uniform,koehler2021uniform,zhou2020uniform,muthukumar2021classification}. 

Understanding benign overfitting in classification problems has been more challenging, since the max-margin predictor to which gradient-based methods are known to converge to (in direction) does not have a closed-form solution. Many of the existing works focus on settings where the max-margin predictor and the (closed-form) least squares predictor coincide (as originally argued in \citet{muthukumar2021classification}). \citet{wang2021benign} and \citet{cao2021risk} use this to study a setting where the two classes are a symmetric mixture of Gaussian (or subgaussian) distributions, without label noise.  \citet{chatterji2021finite} studies a setting where the two classes are a mixture of two product distributions, and with label noise, by studying the trajectory of gradient descent on the training data. \citet{montanari2019generalization} considers classification problem where the inputs are Gaussian, and the labels are generated according to a logistic link function, and derives a formula for the asymptotic prediction error of the max-margin classifier, in a setting where the ratio of the dimension and the sample size converges to some fixed positive limit. Recently, \citet{frei2022benign} managed to prove the existence of benign overfitting in two-layer neural networks with smoothed leaky ReLU activations, assuming the data comes from a mixture of two well-separated distributions. Other works studying benign overfitting and classification include \citet{liang2021interpolating,mcrae2021harmless,poggio2019generalization,thrampoulidis2020theoretical,hu2021understanding,wang2021benign2}.

\section{Preliminaries}\label{sec:preliminaries}

\textbf{Notation.} We use bold-faced letter to denote vectors, and assume that they are in column form unless specified otherwise. Given a vector $\bw\in \reals^d$, we let $w_i$ denote its $i$-th coordinate,  $\bw_{|k}\in \reals^k$ to denote its first $k$ coordinates, and $\bw_{|d-k}\in \reals^{d-k}$ to denote its last $d-k$ coordinates. We also use this notation when the vector already has a subscript for a different purpose, e.g. $\bx_{i|k}$ refers to the first $k$ coordinates of $\bx_i$. Given two vectors $\bu,\bv$ of the same size, $\bu\succeq \bv$ means that $u_i\geq v_i$ for all $i$. We use $[\cdot]_+$ to denote the ReLU function $z\mapsto \max\{0,z\}$. Given a vector $\bv$, $[\bv]_+$ refers to applying the ReLU function entry-wise. $\mathbf{1}(\cdot)$ to denote the indicator function. $[m]$ is shorthand for $\{1,\ldots,m\}$. $I$ denotes the identity matrix (whose size should be clear from context). Given a matrix $A$, $A_{i,j}$ refers to its entry at row $i$ column $j$, $\lambda_{\min}(A)$ refers to its minimal eigenvalue, $\text{Tr}(A)$ to its trace, $\norm{A}$ refers to its spectral norm, and $\norm{A}_F$ refers to the Frobenius norm. It is well known that $\norm{A}\leq \norm{A}_F$.

\textbf{Convergence in probability and the law of large numbers.} For the asymptotic results in our paper, we will often consider a sequence of real-valued random variables $\{X_d\}$ indexed by $d$, and say that they converge in probability to some fixed number $a$ (or $X_d \stackrel{P}{\rightarrow} a$) if $\Pr(|X_d-a|\geq \epsilon)\stackrel{d\rightarrow\infty}{\longrightarrow} 0$ for all $\epsilon\geq 0$. Note that this slightly extends the usual notion of convergence in probability, in that we do not require the random variables to share the same probability space. In addition, we will often utilize the following version of the (weak) law of large numbers for such sequences: 
\begin{lemma}
	If $\sup_d \E[X_d^2]<\infty$, then for any sequence of positive integers $\{m_d\}_{d=1}^{\infty}$ diverging to $\infty$, if we let $\{X_{d,i}\}_{i=1}^{m_d}$ be $m_d$ i.i.d. copies of $X_d$, then $\frac{1}{m_d}\sum_{i=1}^{m_d}X_{d,i}-\E[X_{d}]\stackrel{P}{\rightarrow}0$.
\end{lemma}
The proof immediately follows from Chebyshev's inequality, which implies that since $\E[X_d^2]<\sigma^2$ for some finite $\sigma^2$ and any $d$, then $\Pr\left(\left|\frac{1}{m_d}\sum_{i=1}^{m_d}X_{d,i}-\E[X_d]\right|\geq \epsilon\right)\leq \frac{\sigma^2}{m_d^2\epsilon^2}\stackrel{d\rightarrow\infty}{\longrightarrow} 0$ for all $\epsilon>0$. This law of large numbers  trivially extends to vector-valued and matrix-valued random variables of some fixed dimension $k$.

\textbf{The min-norm predictor.} In linear regression problems with the square loss, we attempt to find a predictor $\bw$ minimizing the expected squared loss $\E_{(\bx,y)}[(\bx^\top\bw-y)^2]$. It is well-known that if we attempt this by  running standard gradient-based methods on the empirical risk objective $\frac{1}{m}\sum_{i=1}^{m}(\bx_i^\top\bw-y_i)^2$ (for some training sample $\{(\bx_i,y_i)\}_{i=1}^{m}$), and the dimension is at least $m$, these methods will generally converge to the (unique) point
\[
\hat{\bw} ~:=~\arg\min_{\bw\in \reals^d}\norm{\bw}:\frac{1}{m}\sum_{i=1}^{m}(\bx_i^\top\bw-y_i)^2=0
\]
(e.g., \cite{zhang2021understanding}). This predictor is also referred to as the minimum-norm interpolating predictor, or simply the min-norm predictor. For our purposes, we will need a somewhat more general result. Concretely, instead of the square loss specifically, suppose we have some loss function $\ell(\bx^\top \bw;y)$, which for any $y$ is minimized at some unique prediction value $\bx^\top \bw$. As before, suppose that we attempt to minimize the expected loss by running a standard gradient-based method over the empirical risk objective, which is now $\frac{1}{m}\sum_{i=1}^{m}\ell(\bx_i^\top\bw;y_i)$, assuming we converge to a globally minimal solution $\hat{\bw}$. The following proposition shows that under a mild assumption, $\hat{\bw}$ is still the minimum-norm interpolating predictor:
\begin{proposition}\label{prop:minnorm}
	Fix a function 
	$
	L(\bw)=\frac{1}{m}\sum_{i=1}^{m} \ell(\bx_i ^\top \bw;y_i)
	$, 
	where each $\ell(\cdot;y_i)$ is a non-negative continuous function which equals $0$ at some unique point denoted as $\ell^{-1}_{y_i}(0)$. Suppose we run an arbitrary iterative training method, that converges to a point $\hat{\bw}$ such that $L(\hat{\bw})=0$ and $\hat{\bw}\in \text{span}\{\bx_1,\ldots,\bx_m\}$. Then $\hat{\bw}$ is the unique point in $\arg\min_{\bw} \norm{\bw}:L(\bw)=0$. 
\end{proposition}
Since gradient-based methods rely on iterative updates along the gradient of $L(\cdot)$ (or gradients of single loss functions $\ell(\bx_i^\top\bw;y_i)$), they generally remain in $\text{span}\{\bx_1,\ldots,\bx_m\}$ assuming we initialize at $\mathbf{0}$, and thus the theorem implies that such methods will converge to minimum-norm solutions that minimize the empirical risk. 

\textbf{Benign overfitting for regression.} For linear prediction problems, benign overfitting is inherently a high-dimensional phenomenon (since when the dimension is fixed, uniform convergence generally occurs). Thus, the most appropriate way to study benign overfitting is to consider a \emph{sequence} of input distributions over $\reals^d$ (indexed by $d$), and study the performance of the learned predictors as both $d$ and the training set size diverge to infinity. Concretely, for the setting studied in the previous paragraph, a standard way to formally define benign overfitting is as follows:
\begin{definition}[Benign Overfitting for the min-norm predictor]\label{def:benign}
	Given a non-negative function $\ell(p;y)$ on $\reals^2$, a sequence of distributions $\{\Dcal_d\}_{d=k+1}^{\infty}$ on $\reals^d\times \reals$ satisfies \emph{benign overfitting}, if there exists a monotonically increasing sequence of positive integers $\{m_d\}_{d=k+1}^{\infty}$ such that the following holds:
	\begin{itemize}
		\item For any sufficiently large $d$, if we sample $m_d$ samples $\{(\bx_i,y_i)\}_{i=1}^{m_d}$ i.i.d. from $\Dcal_d$, then with probability approaching $1$, $\frac{1}{m_d}\sum_{i=1}^{m_d}\ell(\bx_i^\top\bw;y_i))=0$ for some $\bw\in \reals^d$.
		\item Picking 
		$
		\hat{\bw}_d = \arg\min_{\bw} \norm{\bw}:\frac{1}{m}\sum_{i=1}^{m_d}\ell(\bx_i^\top \bw;y_i))=0
		$
		to be the min-norm predictor, 
		and defining $R_{d}(\bw):= \E_{(\bx,y)\sim \Dcal_d}\left[\ell(\bx^\top\bw;y))\right]$,
		it holds that
		$
		\inf_{d} \inf_{\bw\in \reals^d} R_{d}(\bw)>0$ as well as
		$R_{d}(\hat{\bw}_d)-\inf_{\bw\in \reals^d}R_{d}\stackrel{P}{\longrightarrow}0$.
	\end{itemize}
\end{definition}
In other words, the min-norm predictor $\hat{\bw}_d$ is asymptotically optimal, in the sense that its expected loss converges to the best possible expected loss among linear predictors, as the sample size and $d$ diverge to infinity at an appropriate rate. Moreover, this holds despite overfitting, in the sense that the training error (which is $0$) does not converge to the (strictly positive) expected loss. 

\textbf{The max-margin predictor, and benign overfitting for  classification.}
In linear binary classification problems, we consider distributions where the examples $(\bx,y)$ are such that $y\in \{-1,+1\}$, and the predictor (specified by a vector $\bw$) is $\bx\mapsto \text{sign}(\bx^\top \bw)$. In this case, we generally care only about the direction of the predictor $\bw$, and its expected misclassification error, namely $\Pr_{(\bx,y)}(y\bx^\top \bw\leq 0)$.
Whereas in regression, standard methods converge to the minimum-norm interpolating predictor, the characterization in classification is a bit different. Concretely, using standard convex classification losses with exponential tails (such as the logistic or cross-entropy loss), it is by now well-known that gradient-based methods ran on the average loss w.r.t. a given dataset $\{\bx_i,y_i\}_{i=1}^{m}$ converge in direction to the max-margin predictor
\[
\hat{\bw}~=~\arg\min_{\bw\in \reals^d} \norm{\bw}~~:~~ \min_{i\in [m]} y_i \bx_i^\top \bw\geq 1
\]
\citep{soudry2018implicit,ji2020directional}, which by definition achieves zero misclassification error on the dataset. In this setup, we can define benign overfitting as follows: Similar to the case of regression, we need to consider a sequence of distributions $\{\Dcal_d\}_{d=k+1}^{\infty}$ (this time on $\reals^d\times \{-1,+1\}$) and sample sizes $\{m_d\}_{d=k+1}^{\infty}$, which induce a sequence of max-margin predictors $\{\hat{\bw}_d\}_{d=k+1}^{\infty}$ (as defined above) when trained on samples $\{\bx_i,y_i\}_{i=1}^{m_d}$. Letting $R_d(\bw)=\Pr_{(\bx,y)\sim \Dcal_d}(y\bx^\top \bw\leq 0)$, we say that the sequence  $\{\Dcal_d\}_{d=k+1}^{\infty}$ satisfies benign overfitting, if $\hat{\bw}_d$ exists with probability approaching $1$ (as $d\rightarrow \infty$), and 
\begin{equation}\label{eq:classificationbenign}
	\inf_d \inf_{\bw\in \reals^d} R_d(\bw)>0~~~\text{as well as}~~~
	R_d(\hat{\bw}_d)-\inf_{\bw\in \reals^d} R_d(\bw)\stackrel{P}{\longrightarrow}0~.
\end{equation}
Note that this definition is similar to the one we had for regression (Definition \ref{def:benign}), except that $R_d(\cdot)$ is defined with respect to misclassification error, and $\hat{\bw}_d$ is now defined as the max-margin predictor.

\section{Linear Regression with the Square Loss}\label{sec:regression}

We begin by considering the setting of linear regression with the square loss, where our goal is to minimize $\E[(\bx^\top\bw-y)^2]$ with respect to an underlying distribution over $(\bx,y)$. As discussed in the previous section, taking an i.i.d. sample of $m$ training examples, and running standard training methods to convergence, generally results in the min-norm predictor interpolating the data, $\hat{\bw}~=~\arg\min \norm{\bw}~:~ \forall i\in[m]~, \bx_i^\top \bw=y_i$. The question we ask here is whether this predictor is asymptotically consistent and satisfies benign overfitting. 

To motivate our approach, let us first consider the well-specified setting, where $\E[y|\bx]=\bx^\top\bw^*$ for some fixed $\bw^*$, and $\bx$ is zero-mean. This setting was studied in detail in \cite{bartlett2020benign,hastie2019surprises,belkin2020two}. In particular, an important corollary of the results of \citet{bartlett2020benign} is that for benign overfitting to occur, it is \emph{necessary} that for some $k$, the smallest $d-k$ eigenvalues of the covariance matrix $\Sigma$, namely $\lambda_{k+1},\ldots,\lambda_d$, satisfy $m\cdot\frac{\sum_{i>k}\lambda_i^2}{(\sum_{i>k} \lambda_i)^2}\rightarrow 0$. Letting $\Sigma=\Sigma_k\oplus\Sigma_{d-k}$ be the decomposition of $\Sigma$ with respect to the first $k$ and last $d-k$ eigenvalues, this is equivalent to requiring $m\cdot \frac{\norm{\Sigma_{d-k}}_F^2}{\text{Tr}^2(\Sigma_{d-k})}\rightarrow 0$. The following simple lemma implies that this in turn is equivalent to requiring the distribution of $\bx_{|d-k}$ to be high-dimensional, in the sense that inner products of independent copies of $\bx_{|d-k}$ will be very small compared to the ``typical'' size of $\bx_{|d-k}$:
\begin{lemma}\label{lem:orthfrob}
	Let $\bz,\bz'$ be two i.i.d. random vectors in $\reals^n$ such that $\E[\bz\bz^\top]=\Sigma$. Then	$\E[(\bz^\top\bz')^2]~=~\norm{\Sigma}_F^2$. 
	Thus, if $m\cdot \frac{\norm{\Sigma}_F^2}{\text{Tr}^2(\Sigma)}\rightarrow 0$, then $m\cdot\frac{\E[(\bz^\top\bz')^2]}{\E^2[\norm{\bz}^2]}\rightarrow 0$ as well.
\end{lemma}
\begin{proof}
	The proof follows from the observation that $\text{Tr}(\Sigma)=\E[\norm{\bz}^2]$, and that $\E[(\bz^\top\bz')^2]$ equals
	\[
		 \E\left[\left(\sum_{i=1}^{n}z_{i}z'_{i}\right)^2\right]~=~\sum_{i,j=1}^{n}\E[z_{i} z_{j}z'_{i}z'_{j}]
		~=~\sum_{i,j=1}^{n}\E[z_{i} z_{j}]\cdot \E[z'_{i}z'_{j}]	~=~\sum_{i,j=1}^{n}\Sigma_{i,j}^2~=~\norm{\Sigma}_F^2~.
	\]
\end{proof}
Thus, we see that to get benign overfitting even in a well-specified setting, we generally need to consider distributions which have a very high-dimensional component ``spread'' in many directions. Intuitively, this component will be (approximately) mutually orthogonal across different samples, thus allowing the linear predictor to fit the training data, but without significantly affecting the prediction on new samples. A prototypical example to keep in mind is any spherically symmetric distribution in $\reals^{d-k}$, in which case it is not difficult to show that
\[
m\cdot\frac{\E[(\bz^\top\bz')^2]}{\E^2[\norm{\bz}^2]}~=~m\cdot \frac{\norm{\E[\bz\bz^\top]}_F^2}{\E^2[\norm{\bz}^2]}~=~\frac{m}{d-k}~,
\]
which goes to zero as $d\rightarrow \infty$ sufficiently faster than $m$. Moreover, assuming $\norm{\bz}$ does not fluctuate too wildly, the empirical quantity $m\cdot\left(\frac{\bz^\top \bz'}{\norm{\bz}^2}\right)^2$ will strongly concentrate around $0$ as $d\rightarrow \infty$.

Motivated by this, we focus on distributions where the inputs vectors $\bx\in \reals^d$ are such that the first $k$ coordinates $\bx_{|k}$ (for some fixed $k$ independent\footnote{Our proof techniques readily extend to the case of $k$ growing with $d$ at a sufficiently slow rate, but we choose to consider $k$ fixed for simplicity.} of $d$) have some arbitrary distribution, whereas the last $d-k$ coordinates $\bx_{|d-k}$ form a high-dimensional distribution, in the sense that given an i.i.d. sample $\{(\bx_i,y_i)\}_{i=1}^{m}$, it holds with high probability that $\frac{\sup_{i\neq j}|\bx_{i|d-k}^\top \bx_{j|d-k}|}{\min_i \norm{\bx_{i|d-k}}^2}$ converges to $0$ (as the dimension $d$ increases sufficiently fast compared to $m$). We note that assuming such a split to the first $k$ and last $d-k$ coordinates is mostly to simplify the presentation, and is without much loss of generality: Indeed, by a suitable rotation of the coordinate system, all our results extend to distributions which can be split into a high-dimensional distribution in some $d-k$-dimensional subspace, and an arbitrary distribution in the complementary $k$-dimensional subspace. Moreover, variants of this assumption, or related assumptions, are very common in the literature on benign overfitting (e.g., the ``junk features'' model of \citet{zhou2020uniform}, or the ``weak features'' model of \cite{muthukumar2021classification}), although these tend to assume some particular (e.g., Gaussian) distribution on the first $k$ coordinates, whereas we allow that distribution to be rather generic. Finally, we note that we generally do not assume how $y$ depends on $\bx$ (unless specified otherwise). 

\subsection{A Deterministic Perturbation Bound}

Focusing on data samples from such distributions, let us now turn to analyze what form the min-norm predictor takes, as a function of some training data $\{(\bx_i,y_i)\}_{i=1}^{m}$. Intuitively, the analysis rests on viewing the high-dimensional components of the data, $\{\bx_{i|d-k}\}_{i=1}^{m}$, as a perturbation of perfectly orthogonal vectors. To quantify this, we introduce the perturbation matrix $E\in \reals^{m\times m}$, defined as
\begin{equation}\label{eq:Edef}
E_{i,j}=\mathbf{\bx}_{i|d-k}^\top\mathbf{\bx}_{j|d-k}\cdot \mathbf{1}(i\neq j) ~~~\forall i,j\in [m]~.
\end{equation}
Note that if $\{\bx_{i|d-k}\}_{i=1}^{m}$ are perfectly orthogonal, then $E$ is the zero matrix. Also, we will use $\hat{\E}[f(\bx,y)]$ as shorthand for $\frac{1}{m}\sum_{i=1}^{m}f(\bx_i,y_i)$ for any function $f$. The key technical result we will need is the following deterministic perturbation bound, which bounds the distance of $\hat{\bw}_{|k}$ (the first $k$ coordinates) from a certain expression, as well as bounding the norm of the last $d-k$ coordinates $\hat{\bw}_{|d-k}$. 
\begin{theorem}\label{thm:regressiondet}
	Fix some $k\in [d-1]$ and $\{(\bx_i,y_i)\}_{i=1}^{m}\subseteq \reals^d\times \reals$, such that $\{\bx_i\}_{i=1}^{m}$ are linearly independent, 
	$\frac{\norm{E}}{\min_{i\in [m]}\norm{\bx_{i|d-k}}^2}\leq \frac{1}{2}$, and  $\frac{1}{m}\leq  \frac{1}{2}\lambda_{\min}\left(\hat{\E}\left[\frac{\bx_{|k}\bx_{|k}^\top}{\norm{\bx_{|d-k}}^2}\right]\right)$. 
	Then the min-norm predictor $\hat{\bw}$ exists and satisfies
	\begin{align*}
	&\left\|\hat{\bw}_{|k}-\left(\hat{\E}\left[\frac{\bx_{|k}\bx_{|k}^\top}{\norm{\bx_{|d-k}}^2}\right]\right)^{-1} \hat{\E}\left[\frac{y\bx_{|k}}{\norm{\bx_{|d-k}}^2}\right]\right\|\\
	&~~~~~~\leq~\frac{2\left\|\hat{\E}\left[\frac{y\bx_{|k}}{\norm{\bx_{|d-k}}^2}\right]\right\|}{\lambda_{\min}\left(\hat{\E}\left[\frac{\bx_{|k}\bx_{|k}^\top}{\norm{\bx_{|d-k}}^2}\right]\right)^2}\cdot \frac{1}{m}~+~\frac{2\sqrt{\hat{\E}[\norm{\bx_{|k}}^2]\cdot \hat{\E}[y^2]}\cdot }{\min_{i\in [m]}\norm{\bx_{i|d-k}}^4}\cdot m\cdot\norm{E}~,
	\end{align*}
	as well as
	\[
	\norm{\hat{\bw}_{|d-k}}~\leq~ \frac{\sqrt{\hat{\E}[\norm{\bx_{|d-k}}^2]\cdot \hat{\E}[y^2]}}{\min_{i\in [m]}\norm{\bx_{i|d-k}}^2}\cdot\left(1+\frac{2\norm{E}}{\min_{i\in [m]}\norm{\bx_{i|d-k}}^2}\right)\cdot m~. 
	\]
\end{theorem}

The key take-away from this theorem is as follows: Assuming various ratios and empirical moments of the dataset are bounded, then 
\[
\left\|\hat{\bw}_{|k}- \left(\hat{\E}\left[\frac{\bx_{|k}\bx_{|k}^\top}{\norm{\bx_{|d-k}}^2}\right]\right)^{-1} \hat{\E}\left[\frac{y\bx_{|k}}{\norm{\bx_{|d-k}}^2}\right]\right\|~\leq~ \Ocal\left(\frac{1}{m}+m\cdot\norm{E}\right)~,
\]
where $E\rightarrow 0$ as the inner products between pairs of vectors in $\{\bx_{i|d-k}\}_{i=1}^{m}$ go to zero. Assuming this convergence to zero is sufficiently fast compared to $m$, that $m\rightarrow \infty$, and that the law of large numbers hold, we get that
\[
\hat{\bw}_{|k}~\rightarrow~\left(\E\left[\frac{\bx_{|k}\bx_{|k}^\top}{\norm{\bx_{|d-k}}^2}\right]\right)^{-1} \E\left[\frac{y\bx_{|k}}{\norm{\bx_{|d-k}}^2}\right]~.
\]
As to $\hat{\bw}_{|d-k}$, we effectively bound its norm by $\Ocal(m)$, which scales with $m$ but not with the dimension. If the input distribution on the last $d-k$ coordinates is sufficiently high-dimensional, this implies that given a new sample $\bx$, the contribution of $\bx_{|d-k}^\top\hat{\bw}_{|d-k}$ to the predicted value (namely $\bx^\top\hat{\bw}=\bx_{|k}^\top \hat{\bw}_{|k}+\bx_{|d-k}^\top\hat{\bw}_{|d-k}$) is asymptotically negligible. Thus, the asymptotic expression of $\hat{\bw}_{|k}$ eventually determines the behavior of the learned predictor. 

Before continuing, let us provide an informal and partial proof sketch, explaining where the approximate expression for $\hat{\bw}_{|k}$ in \thmref{thm:regressiondet} comes from. To that end, let $X$ be a matrix whose $i$-th row is $\bx_i$, and $\by=(y_1,\ldots,y_m)$. By standard results, $\hat{\bw}=\arg\min_{\bw}\norm{\bw}:X\bw=\by$ has the closed-form expression $X^\top(XX^\top)^{-1}\by$. Letting $X_{|k}$ be the first $k$ columns of $X$, it follows that $\hat{\bw}_{|k}=X_{|k}^\top((XX^\top)^{-1}\by)=X_{|k}^\top(X_{|k}X_{|k}^\top+X_{|d-k}X_{|d-k}^\top)\by$. Suppose for simplicity that $\{\bx_{i|d-k}\}_{i=1}^{m}$ are precisely orthogonal (so that $E=0$ in the theorem above, and $X_{|d-k}X_{|d-k}^\top$ equals a diagonal matrix $D$). As a result, we get $\hat{\bw}_{|k}=X_{|k}^\top\left(X_{|k} X_{|k}^\top+D\right)^{-1}\by$. By the Woodbury matrix identity and some algebraic manipulations, this equals \\
$\left(I+X_{|k}^\top D^{-1} X_{|k}\right)^{-1}X_{|k}^\top D^{-1}\by$, or equivalently, 
\[
	\left(\frac{1}{m}I+\frac{1}{m}X_{|k}^\top D^{-1} X_{|k}\right)^{-1}\left(\frac{1}{m} X_{|k}^\top D^{-1}\by\right)
	~=~\left(\frac{1}{m}I +\hat{\E}\left[\frac{\bx_{|k}\bx_{|k}^\top}{\norm{\bx_{|d-k}}^2}\right]\right)^{-1}\cdot \hat{\E}\left[\frac{y\bx_{|k}}{\norm{\bx_{|d-k}}^2}\right]~,
\]
which approaches $\left(\hat{\E}\left[\frac{\bx_{|k}\bx_{|k}^\top}{\norm{\bx_{|d-k}}^2}\right]\right)^{-1}\cdot \hat{\E}\left[\frac{y\bx_{|k}}{\norm{\bx_{|d-k}}^2}\right]$ as $m$ increases. 

\subsection{Asymptotic Characterization of the min-norm predictor}

Let us now turn to show how \thmref{thm:regressiondet} can lead to a formal asymptotic characterization of the min-norm predictor $\hat{\bw}$, in a statistical setting where the training data is sampled from some underlying distribution. To do so, we will need to impose assumptions on the distribution, which ensure that the perturbation matrix $E$ from \thmref{thm:regressiondet} indeed converges to $0$, and that the various quantities in the bounds are well-behaved. One such set of sufficient conditions is the following:

\begin{assumption}\label{assump:regression}
	 Suppose $\{\Dcal_d\}_{d=k+1}^{\infty}$ is a sequence of distributions on $\reals^d\times \reals$, and $\{m_d\}_{d=k+1}^{\infty}$ a monotonically increasing sequence of positive integers diverging to $\infty$, such that the following hold:
	\begin{enumerate}
	\item Letting $\E_d$ be shorthand for $\E_{(\bx,y)\sim \Dcal_d}$, it holds that 
	\[
	\sup_d \max\left\{\E_d[\norm{\bx}^4]~,~\E_d[y^4]~,~\E_d\left[\frac{\norm{y\bx_{|k}}^2}{\norm{\bx_{|d-k}}^4}\right]
	~,~\E_d\left[\frac{\norm{\bx_{|k}}^4}{\norm{\bx_{|d-k}}^4}\right]\right\}<\infty~,
	\]
	and  $\inf_d \lambda_{\min}\left(\E_d\left[\frac{\bx_{|k}\bx_{|k}^\top}{\norm{\bx_{|d-k}}^2}\right]\right)>0$.
	\item If we sample $m_d$ i.i.d. samples $\{\bx_i,y_i\}_{i=1}^{m_d}$ from $\Dcal_d$, then with probability approaching $1$, $\{\bx_i\}_{i=1}^{m_d}$ are linearly independent, and  $\min_{i\in[m_d]}\norm{\bx_{i|d-k}}$ is at least some $c>0$ independent of $d$. 
	\item $m_d\cdot \norm{E}\stackrel{P}{\rightarrow} 0$, where $E$ is as defined in \eqref{eq:Edef}.
	\item $m_d^2\cdot\norm{\E_{d}[\bx_{|d-k}\bx_{|d-k}^\top]}\rightarrow 0$~.
\end{enumerate}	 
\end{assumption}

Since the assumptions are rather technical, let us provide one simple example to keep in mind, which satisfies the above:

\begin{example}\label{example:regression}
Suppose that $\lim_{d\rightarrow \infty}\frac{m_d^3\log(d)}{d}=0$, and that $\Dcal_d$ is defined as follows: $(\bx_{|k},y)$ has some fixed distribution (independent of $d$), with bounded moments up to order $4$ and such that $\E[\bx_{|k}\bx_{|k}^\top]$ is positive definite; And $\bx_{|d-k}$ is an independent zero-mean Gaussian  with covariance matrix $\frac{1}{d-k}\cdot I$.\\
In this case, $\norm{\bx_{|d-k}}$ strongly concentrates around $1$ as $d\rightarrow\infty$ increases, and $\{\bx_{i|d-k}\}_{i=1}^{m_d}$ are linearly independent with probability $1$, hence conditions $1$ and $2$ in the assumption clearly holds. As to conditions $3$ and $4$,  letting $Z\in\reals^{m\times (d-k)}$ be the matrix whose rows are $\bx_{i|d_k}$, and $D$ the diagonal of $ZZ^\top$, we have $\norm{E}=\norm{ZZ^\top-D}\leq \norm{ZZ^\top-I}+\norm{I-D}$, which is at most $\Ocal(\sqrt{m_d\log(d)/d})$ with probability converging to $1$ as $d\rightarrow \infty$ (see for example \citep{zhu2012short}). Also, clearly $ \norm{\E[\bx_{|d-k}\bx_{|d-k}^\top]}=\frac{1}{d-k}$. Combined with the assumption $\lim_{d\rightarrow \infty}\frac{m_d^3\log(d)}{d}=0$, conditions 3,4 follow.
\end{example}

We remark that in the example, we require $\lim_{d\rightarrow \infty}\frac{m_d^3\log(d)}{d}=0$, which is a stronger assumption on the scaling of $d$ vs. $m_d$ compared to previous work on linear regression (which usually consider $m_d/d\rightarrow 0$ under similar distributional assumptions). On the flip side, the proof technique allows us to analyze more general settings which go beyond well-specified linear regression. 

In any case, we emphasize that Assumption \ref{assump:regression} applies far more broadly than Example \ref{example:regression}: For instance, it generally applies to any spherically-symmetric distribution of $\bx_{|d-k}$ (possibly dependent on $\bx_{|k},y$) so that $\norm{\bx_{|d-k}}$ is bounded (or at least concentrated) in some fixed interval bounded away from $0$. Also, Assumption \ref{assump:regression} itself is not the most general possible, in the sense that it focuses on situations where $\norm{\bx},y$ and $\frac{1}{\norm{\bx_{|d-k}}}$ are scaled so that they are essentially bounded independent of $d$. Moreover, using \thmref{thm:regressiondet} above one can analyze even more general situations: For example, that the data norm scales with $d$, while only bounding various ratios between relevant quantities. 

Under Assumption \ref{assump:regression}, let us now proceed to formally state our asymptotic characterization of the min-norm predictor:

\begin{theorem}\label{thm:regression}
	Suppose $\{\Dcal_d\}_{d=k+1}^{\infty}$ and $\{m_d\}_{d=k+1}^{\infty}$ satisfy Assumption \ref{assump:regression}. For any $d$, let  $\hat{\bw}_d=\arg\min\norm{\bw}:\forall i\in [m_d],~ \bx_i^\top \bw=y_i$ be the min-norm predictor w.r.t. a training set of size $m_d$ sampled i.i.d. from $\Dcal_d$. Then as $d\rightarrow \infty$, $\hat{\bw}_d$ exists with probability approaching $1$, and satisfies
	\[
\left\|\hat{\bw}_{d|k}-\left(\E_d\left[\frac{\bx_{|k}\bx_{|k}^\top}{\norm{\bx_{|d-k}}^2}\right]\right)^{-1} \E_d\left[\frac{y\bx_{|k}}{\norm{\bx_{|d-k}}^2}\right]\right\|\stackrel{P}{\longrightarrow} 0~~~\text{and}~~~
\E_{d}\left[(\bx^\top \hat{\bw}_d-\bx_{|k}^\top \hat{\bw}_{d|k})^2\right]\stackrel{P}{\longrightarrow} 0~.
\]
\end{theorem}

\subsection{Implications for benign overfitting in regression}

Having established this asymptotic characterization of $\hat{\bw}$, we now turn to discuss its implications to benign overfitting in linear regression. Our bottom-line message is that in general, $\hat{\bw}_d$ is asymptotically \emph{not} an optimal predictor, and hence benign overfitting will not occur. 

To see this, consider any sequence of distributions $\{\Dcal_d\}$ as in \thmref{thm:regression}. For any $\Dcal_d$, the expected loss has the form
\[
R_d(\bw)~:=\E_{(\bx,y)\sim \Dcal_d}[(\bx^\top \bw-y)^2]~=~ \bw^\top \E_d[\bx\bx^\top]\bw-2\E_d[y\bx^\top]\bw+\E_d[y^2]~.
\]
Assuming $\E_d[\bx\bx^\top]$ is positive definite, it follows that the unique minimizer equals $\bw=\E[\bx\bx^\top]^{-1}\E[y\bx]$, and in particular
\[
\bw_{|k}~=~ \left(\E_d[\bx\bx^\top]^{-1}\right)_{k}\E_d[y\bx]~,
\]
where $\left(\E_d[\bx\bx^\top]^{-1}\right)_{k}$ refers to the first $k$ rows of the inverse covariance matrix. Thus, for benign overfitting to occur, we need that the min-norm predictor $\hat{\bw}_d$ will be such that $\hat{\bw}_{d|k}$ equals this expression, at least asymptotically. However, \thmref{thm:regression} implies that $\hat{\bw}_{d|k}$ is asymptotically biased towards a different expression, namely $\left(\E_d\left[\frac{\bx_{|k}\bx_{|k}^\top}{\norm{\bx_{|d-k}}^2}\right]\right)^{-1} \E_d\left[\frac{y\bx_{|k}}{\norm{\bx_{|d-k}}^2}\right]$. Thus, unless the two expressions somehow exactly coincide, there is no reason to believe that benign overfitting to occur, even though the covariance structure of the inputs $\bx$ can be a textbook case of amenability to benign overfitting (in terms of the conditions of \thmref{thm:regression} or previous papers on benign overfitting in regression). The following example illustrates this:

\begin{example}
In the setting of \thmref{thm:regression}, suppose $k=1$, $x_{1}$ (the first coordinate of $\bx$) is uniform on the interval $[-a,a]$ for some arbitrary $a>0$, $y=\exp(x_1)$, and for all $j\in\{2,\ldots,d\}$, $x_j = \sqrt{\frac{y}{d-1}}\cdot r_j$, where each $r_j$ is an independent standard Gaussian random variable. Then $\E[x_1 x_j]=\E[y x_j]=0$ for all $j>1$, and therefore
\begin{align*}
R_d(\bw)&=~\E_{(\bx,y)\sim \Dcal_d}[(\bx^\top \bw-y)^2]~=~ \bw^\top \E[\bx\bx^\top]\bw-2\E[y\bx]^\top \bw+\E[y^2]\\
&=~ \E[x_1^2]\cdot w_1^2+\frac{\E[y]}{d-1}\cdot\sum_{j=2}^{d}w_j^2-2\E[yx_1]\cdot w_1+\E[y^2]~.
\end{align*}
By differentiating the above w.r.t. $\bw$, it is easily verified that $R_d(\cdot)$ achieves a minimal value only when $w_j=0$ for all $j>1$, and
$
w_1= \frac{\E[yx_{1}]}{\E[x_{1}^2]}~=~ \frac{\E[\exp(x_1)x_{1}]}{\E[x_1^2]},
$
which is a strictly positive number dependent only on $a$. However, by \thmref{thm:regression} and standard concentration results for the Gaussian distribution, the first coordinate of $\hat{\bw}_d$ converges in probability to the different value $\E\left[\frac{yx_1}{\exp(x_1)}\right]/~\E\left[\frac{x_1^2}{\exp(x_1)}\right]~=~\frac{\E[x_1]}{\E[x_1^2\exp(-x_1)]}~=~0$. It follows that $R_d(\hat{\bw}_d)-\inf_{\bw} R_d(\bw)$ is lower bounded by a positive number independent of $d$, and therefore we do not have benign overfitting.
\end{example}

The reader familiar with previous literature might wonder how this can possibly accord with previous results (such as \citet{bartlett2020benign}), which show that benign overfitting \emph{does} occur for linear regression with the square loss, under the kind of input distributions we study here. The reason is that these results assume a well-specified setting, where $\E[y|\bx]=\bx^\top \bw^*$ for some fixed $\bw^*$ (see for example Assumption 4 in Definition 1 of \citet{bartlett2020benign}). In the example above, this does not hold, since $\E[y|\bx] = \exp(x_1)$ is not a linear function of $\bx$. Had we been in a well-specified setting with $\E[y|\bx]=\bx_{|k}^\top \bw^*_{|k}$ for some $\bw^*$, benign overfitting \emph{would} generally occur, because then we have that
\[
\left(\E\left[\frac{\bx_{|k}\bx_{|k}^\top}{\norm{\bx_{|d-k}}^2}\right]\right)^{-1}\cdot \E\left[\frac{y\bx_{|k}}{\norm{\bx_{|d-k}}^2}\right]~=~
\left(\E\left[\frac{\bx_{|k}\bx_{|k}^\top}{\norm{\bx_{|d-k}}^2}\right]\right)^{-1}\cdot \E\left[\frac{\bx_{|k}\bx_{|k}^\top}{\norm{\bx_{|d-k}}^2}\right]\bw^*_{|k}~=~
\bw^*_{|k}~,
\]
which now coincides with the optimal solution on the first $k$ coordinates. However, a well-specified distribution is the exception rather than the rule in practice. 

\section{Linear Regression Beyond the Square Loss}\label{sec:beyond}

In the previous section, we studied linear regression with the square loss, with our main conclusion being that benign overfitting should not be expected in general, beyond well-specified distributions. In this section, we study what happens if we do focus on well-specified distributions, but consider more general regression problems (beyond linear regression with the square loss). We will see that here again, benign overfitting can generally fail to hold. 

Concretely, suppose that instead of the square loss, we have some non-negative loss function $\ell(\bx^\top\bw;y)$, which is relevant for regression in the sense that for any $y$, it equals $0$ at some unique value $\ell^{-1}_y(0)$. As discussed in \secref{sec:preliminaries} and \propref{prop:minnorm}, we still expect standard gradient-based training methods to converge to a minimum-norm interpolating predictor, assuming they manage to drive the average loss on the training set to $0$: Namely, given a dataset $\{\bx_i,y_i\}_{i=1}^{m}$, the unique point in $\arg\min_{\bw}\norm{\bw}:\frac{1}{m}\sum_{i=1}^{m}\ell(\bx_i^\top\bw;y_i)=0$. The question now is whether this predictor enjoys benign overfitting.

To answer this question, one option is to try and repeat the analysis from the previous section, depending on the choice of $\ell(\cdot)$. However, we will take a different approach, which allows us to study this using the results we already developed for linear regression with the square loss. Our crucial observation can be phrased as the following, very simple lemma:
\begin{lemma}\label{lem:minnorm}
	Given a dataset $\{\bx_i,y_i\}_{i=1}^{m}$ and non-negative loss $\ell(\cdot)$ as above,
	\[
	\arg\min_{\bw} \norm{\bw}:\frac{1}{m}\sum_{i=1}^{m}\ell(\bx_i^\top\bw;y_i)=0~~~~\text{equals}~~~
	\arg\min_{\bw} \norm{\bw}:\frac{1}{m}\sum_{i=1}^{m}(\bx_i^\top \bw-\ell^{-1}_{y_i}(0))^2=0~.
	\]
\end{lemma}
The proof of the lemma trivially follows from the observation that if $\frac{1}{m}\sum_{i=1}^{m}\ell(\bx_i^\top\bw;y_i)=0$, then $\ell(\bx_i^\top\bw;y_i)=0$ for all $i$, hence by assumption, $(\bx_i^\top \bw-\ell^{-1}_{y_i}(0))^2=0$ for all $i$. 

The lemma implies that the same method that converges to the minimum-norm minimizer of the average loss w.r.t. $\ell(\cdot)$, also simultaneously converges to the minimum-norm minimizer of the average loss w.r.t. the square loss with target values $\ell^{-1}(y_i)$. However, assuming that the training set $\{(\bx_i,y_i)\}_{i=1}^{m}$ is sampled i.i.d. from some underlying distribution, it is evident that they represent empirical risk minimization of two distinct statistical learning problems: One being minimizing $\E[\ell(\bx^\top\bw;y)]$, and the other minimizing $\E[(\bx^\top\bw-\ell^{-1}_{y}(0))^2]$. In general, these are different learning problems, with distinct optima with respect to the underlying data distribution, yet the returned $\hat{\bw}$ is exactly the same one. Thus, if we have benign overfitting in one problem (with the trained predictor $\hat{\bw}$ having near-minimal expected loss), we generally \emph{cannot} expect to have benign overfitting in the other problem. Thus, \emph{the very fact that we can show benign overfitting in settings such as well-specified linear regression with the square loss, precludes the possibility of having benign overfitting in other learning problems}.

In what follows, we exemplify this observation on two types of well-specified regression problems. The first setting we study is a generalized linear model. Concretely, we consider predictors of the form $\bx\mapsto \sigma(\bx^\top \bw)$, where $\bw$ is the parameter vector and $\sigma(\cdot)$ is some strictly monotonic non-linear function, and assume that $\E[y|\bx]=\sigma(\bx^\top \bw^*)$ for some $\bw^*$. In the context of neural networks, this can also be viewed as training a single neuron using some nonlinear activation function $\sigma(\cdot)$. In this setting, standard gradient-based methods trained on the average square loss (i.e., $\min_{\bw}\frac{1}{m}\sum_{i=1}^{m}(\sigma(\bx_i^\top\bw)-y_i)^2$) will indeed generally converge to the min-norm predictor, namely $\arg\min_{\bw}\norm{\bw} ~:~ \frac{1}{m_d}\sum_{i=1}^{m_d}(\sigma(\bx_i^\top\bw)-y_i)^2=0$ (see for example the proof of \citet[Thm. 3.2]{yehudai2020learning}, combined with \propref{prop:minnorm}). The following corollary of our previous results implies that for just about any choice of input distribution on the first $k$ coordinates, and just about any choice of a strictly monotonic non-linear $\sigma(\cdot)$, we generally \emph{cannot} expect benign overfitting to occur, even if the model is well-specified: 
\begin{corollary}\label{cor:glm}
	Suppose that $\sigma:\reals\rightarrow\reals$ is a function whose inverse $\sigma^{-1}(\cdot)$ exists and is Lipschitz continuous. Consider any sequence of distributions $\{\Dcal_d\}_{d=k+1}^{\infty}$ and integers $\{m_d\}_{d=k+1}^{\infty}$ satisfying Assymption \ref{assump:regression}, such that for any $d$ and $(\bx,y)\sim\Dcal_d$, $y = \sigma(\bx_{|k}^\top \bw^*)+\xi$ for some fixed $\bw^*\in \reals^k$ and random variable $\xi$. 
	Then the min-norm predictor $\hat{\bw}_d$ satisfies
	\[
	\left\|\hat{\bw}_{d|k}~-~\left(\E_d\left[\frac{\bx_{|k}\bx_{|k}^\top}{\norm{\bx_{|d-k}}^2}\right]\right)^{-1}\cdot \E_d\left[\frac{\sigma^{-1}\left(\sigma(\bx_{|k}^\top\bw^*)+\xi\right)\bx_{|k}}{\norm{\bx_{|d-k}}^2}\right]\right\|~\stackrel{P}{\longrightarrow}~0
	\]
	and $\E_{(\bx,y)\sim \Dcal_d}\left[(\bx^\top \hat{\bw}_d-\bx_{|k}^\top \hat{\bw}_{d|k})^2\right]~\stackrel{P}{\longrightarrow}~0$
\end{corollary}
The corollary follows immediately from the observation that since $\sigma$ is invertible, $\hat{\bw}_d$ is also the minimum-norm minimizer of $\frac{1}{m_d}\sum_{i=1}^{m_d}(\bx_i^\top\bw-\sigma^{-1}(y_i))^2=0$, and that the moment conditions in Assumption \ref{assump:regression} are still satisfied if we replace $y$ by $\sigma^{-1}(y)$ (since $|\sigma^{-1}(y)|\leq c_{\sigma}(1+|y|)$ for some $c_{\sigma}>0$ dependent only on $\sigma$). Applying \thmref{thm:regression} on $\hat{\bw}_d$ with $y$ replaced by $\sigma^{-1}(y)$, we get that
\[
\left\|\hat{\bw}_{d|k}~-~
\left(\E_d\left[\frac{\bx_{|k}\bx_{|k}^\top}{\norm{\bx_{|d-k}}^2}\right]\right)^{-1} \E_d\left[\frac{\sigma^{-1}(y)\bx_{|k}}{\norm{\bx_{|d-k}}^2}\right]
\right\|~\stackrel{P}{\longrightarrow}~0~,
\]
and plugging in $y=\sigma(\bx_{|k}^\top)+\xi$ results in the theorem.  

When $\xi$ is independent zero-mean noise, and $\sigma(\cdot)$ (and hence $\sigma^{-1}(\cdot)$) is a linear function, then the asymptotic expression for $\hat{\bw}_{d|k}$ in the theorem above reduces to $\bw^*$, which is indeed the optimal vector we would hope to converge to. However, when $\sigma(\cdot)$ is nonlinear, the expression is not $\bw^*$ in general, and hence we do not get asymptotic consistency. To give just one simple example, suppose that  $\norm{\bx_{|d-k}}=1$ with probability $1$, $\sigma(0)=0$, $\bw^*=\mathbf{0}$ and $\E[\bx_{|k}]\neq\mathbf{0}$. In this case the asymptotic expression for $\hat{\bw}_{d|k}$ in the theorem reduces to, $\left(\E[\bx_{|k}\bx_{|k}^\top]\right)^{-1}\E[\bx_{|k}]\cdot\E[\sigma^{-1}(\xi)]$. For this to equal $\bw^*$ (namely $\mathbf{0}$), we need that $\E[\sigma^{-1}(\xi)]=0$. However, since $\sigma(\cdot)$ (and hence $\sigma^{-1}(\cdot)$) is non-linear, the equation above will not hold for ''most`` zero-mean distributions of $\xi$. In other words, even if we fix the input distribution, then just by playing around with the distribution of the noise term $\xi$, we can easily encounter situations where benign overfitting does not hold. Concretely, the following lemma (whose proof is in the appendix) shows that \emph{no} nonlinear $\sigma(\cdot)$ can possibly satisfy $ \E[\sigma^{-1}(\xi)]~=~0$ for all zero-mean distributions:
\begin{lemma}\label{lem:linear}
	Suppose that $\sigma^{-1}(\cdot)$ is a function on $\reals$ such that $\E[\sigma^{-1}(\xi)]=0$ for all zero-mean random variables $\xi$ with support of size at most $2$. Then $\sigma^{-1}(\cdot)$ (and hence $\sigma(\cdot)$) must be a homogeneous linear function (that is, $\exists c\in \reals~\text{s.t.}~\forall z\in \reals,~ \sigma^{-1}(z)=cz$).
\end{lemma}

Next, we go back to linear regression, but now assume that we use some convex loss which is not necessarily the square loss (say, the absolute loss). Here again, standard gradient methods trained on the average loss will generally converge to a min-norm interpolating predictor (thanks to convexity and \propref{prop:minnorm}). However, we cannot expect benign overfitting to occur in general for this predictor:
\begin{corollary}\label{cor:nonsquare}
	Consider any sequence of distributions $\{\Dcal_d\}_{d=k+1}^{\infty}$ and integers $\{m_d\}_{d=k+1}^{\infty}$ satisfying Assumption \ref{assump:regression}. Suppose we use the loss function $\ell(\bx^\top\bw;y)=f(\bx^\top\bw-y)$ for some non-negative function $f$ which has a unique root at $0$. Then the min-norm predictor $\hat{\bw}$ satisfies $\E_{(\bx,y)\sim \Dcal_d}\left[(\bx^\top \hat{\bw}_d-\bx_{|k}^\top \hat{\bw}_{d|k})^2\right]\stackrel{P}{\longrightarrow}0$ and
	\[
	\left\|\hat{\bw}_{d|k}-\left(\E_d\left[\frac{\bx_{|k}\bx_{|k}^\top}{\norm{\bx_{|d-k}}^2}\right]\right)^{-1}\cdot \E_d\left[\frac{y\bx_{|k}}{\norm{\bx_{|d-k}}^2}\right]\right\|~\stackrel{P}{\longrightarrow}~0~.
	\]
\end{corollary}
The proof is immediate from observing that $\hat{\bw}_d$ is also $\arg\min \norm{\bw}:\frac{1}{m_d}\sum_{i=1}^{m_d}(\bx_i^\top \bw-y_i)^2=0$, and applying \thmref{thm:regression} on this related linear regression problem. Crucially, note that $\hat{\bw}_{d|k}$ has the same asymptotic characterization as if we have used the square loss, and there is no reason to believe that this is also an optimal solution (w.r.t. the first $k$ coordinates) of $\E[f(\bx^\top \bw-y)]$ when $f(\cdot)$ is not the square loss. Let us illustrate this with a simple example:
\begin{example}
	In the setting of \corollaryref{cor:nonsquare}, suppose $f(z)=|z|$ is the absolute loss, $k=1$, $x_{1}=1$ with probability $1$, and $y = x_{1}+\xi$ for some independent random variable $\xi$. Then the first coordinate of $\hat{\bw}_d$ converges in probability to $\E[yx_1]/\E[x_1^2]=1+\E[\xi]$. However, the expected absolute loss is
	$
	R_d(\bw)~=~ \E[|\bx^\top \bw-y|]~=~ \E\left[\left|w_1+\sum_{j=2}^{d}x_jw_j-(1+\xi)\right|\right]
	$, 
	which is easily verified to be minimized only when $w_1=1+\text{med}(\xi)$ (where $\text{med}(\xi)$ is the median of $\xi$). Thus, whenever $\text{med}(\xi)\neq \E[\xi]$ (which generally occurs when $\xi$ has a non-symmetric distribution),  $R_d(\hat{\bw}_d)-\inf_{\bw} R_d(\bw)$ does not converge to $0$, and we do not have benign overfitting.
\end{example}  

\begin{remark}[Implicit bias towards a weighted square loss problem]
	In \thmref{thm:regression}, for linear regression with the square loss, we saw that $\hat{\bw}_{d|k}$ asymptotically equals 
	\[
	\left(\E_d\left[\frac{\bx_{|k}\bx_{|k}^\top}{\norm{\bx_{|d-k}}^2}\right]\right)^{-1} \E_d\left[\frac{y\bx_{|k}}{\norm{\bx_{|d-k}}^2}\right]~.
	\]
	This can be equivalently seen as the minimum-norm optimum of the objective function
	\[
	\E_{d} \left[\left(\frac{\bx^\top}{\norm{\bx_{|d-k}}}\bw-\frac{y}{\norm{\bx_{|d-k}}}\right)^2\right]~.
	\]
	In other words, even though $\hat{\bw}$ minimizes $\frac{1}{m}\sum_{i=1}^{m}(\bx_i^\top \bw-y_i)^2$, and should asymptotically minimize $\E[(\bx^\top \bw-y)^2]$ for benign overfitting to occur, its first $k$ coordinates actually optimize a \emph{weighted} version of this problem, where both $\bx,y$ are scaled down by $\norm{\bx_{|d-k}}$. This can be explained via the approach developed in this section: It trivially holds that $\hat{\bw}=\arg\min_{\bw} \norm{\bw}:\frac{1}{m}\sum_{i=1}^{m}(\bx_i^\top\bw-y_i)^2=0$ also equals
	\[
	\arg\min_{\bw} \norm{\bw}:\frac{1}{m}\sum_{i=1}^{m}\left(\frac{\bx_i^\top}{\norm{\bx_{i|d-k}}} \bw-\frac{y_i}{\norm{\bx_{i|d-k}}}\right)^2=0~,
	\]
	and in terms of asymptotic behavior, it turns out that $\hat{\bw}$ is actually ``consistent'' with respect to the statistical problem associated with the latter, weighted loss function, and not the former unweighted one.
\end{remark}

\section{Linear Binary Classification}\label{sec:classification}

The results in the previous section suggest that many natural extensions of well-specified linear regression with the square loss will generally not satisfy benign overfitting. These were all regression problems, where to get low loss the prediction value must be close to some optimal value. 

In this section, we turn to consider binary linear classification setups, where we only care about the sign of $\bx^\top \bw$ rather than its exact value, and see that the situation there is much more favorable. As in the case of regression, we will focus on input distributions which can be decomposed to some arbitrary distribution on the first $k$ coordinates, and a high-dimensional distribution on the last $d-k$ coordinates (for example, a spherically symmetric distribution).

As discussed in \secref{sec:preliminaries}, whereas for regression we care about the minimum-norm interpolating predictor, for classification we care about the max-margin predictor, $\hat{\bw}=\arg\min_{\bw}:\min_{i\in [m]}y_i\bx_i^\top\bw\geq 1$. To study benign overfitting for such problems, we need an asymptotic characterization of $\hat{\bw}$, similar to what we have done for regression. However, this might seem difficult, since unlike the min-norm predictor, the max-margin predictor does not have a closed-form expression. In fact, many previous analyses of benign overfitting in classification resorted to additional assumptions which make the max-margin predictor coincide with the min-norm solution, $\arg\min_{\bw}\norm{\bw}:\frac{1}{m_d}\sum_{i=1}^{m_d}(\bx_i^\top\bw-y_i)^2=0$. We take a different route, which applies  even when the max-margin and min-norm solutions do not coincide: We show that at least for distributions falling within our framework, the first $k$ coordinates of $\hat{\bw}_d$ asymptotically minimize the expectation of a (weighted) \emph{squared hinge loss} on those coordinates, namely $\arg\min_{\bv\in \reals^k}\E\left[\frac{[1-y\bx_{|k}^\top \bv]_+^2}{\norm{\bx_{|d-k}}^2}\right]$. As in the case of regression, we will first show how to derive this using a deterministic perturbation bound (depending on the extent to which the high-dimensional components in the data are far from being perfectly orthogonal), followed by a probabilistic asymptotic characterization of the max-margin predictor, and finally discuss its implications. 

\subsection{A Deterministic Perturbation Bound}

Recall that our analysis for regression relied on the assumption that the input distribution on the last $d-k$ coordinates is high-dimensional, in the sense that $\bx_{i|d-k}^\top \bx_{j|d-k}\approx 0$ for $i\neq j$. With this scenario in mind, we present the following deterministic perturbation bound, which characterizes $\hat{\bw}_{|k}$ and the norm of $\hat{\bw}_{|d-k}$ when such inner products are small. As in the case of regression, we utilize a perturbation matrix $E\in \reals^{m\times m}$, which is now defined as
\begin{equation}\label{eq:Edef2}
E_{i,j}~=~y_iy_j\mathbf{\bx}_{i|d-k}^\top\mathbf{\bx}_{j|d-k}\cdot \mathbf{1}(i\neq j)~~\forall i,j\in [m]~,
\end{equation}
and use $\hat{\E}[f(\bx,y)]$ as shorthand for $\frac{1}{m}\sum_{i=1}^{m}f(\bx_i,y_i)$ for any function $f$.

\begin{theorem}\label{thm:classificationdet}
	Fix some $k\in [d-1]$ and $\{(\bx_i,y_i)\}_{i=1}^{m}\subseteq \reals^d\times \{-1,+1\}$. Suppose  $\epsilon_0:=\frac{2\norm{E}\cdot\max_{i\in[m]}\norm{\bx_{i|d-k}}^2}{\min_{i\in[m]}\norm{\bx_{i|d-k}}^4}\leq \frac{1}{2}$. Then the max-margin predictor
	$\hat{\bw}=\arg\min \norm{\bw}: \forall i\in [m], y_i\bx_i^\top \bw\geq 1$ (assuming it exists) satisfies the following:
	\[
	\hat{\bw}_{|k}~=~ \arg\min_{\bv\in \reals^k} ~(1+\epsilon_{\bv})\cdot\hat{\E}\left[\frac{[1-y\bx_{|k}^\top \bv]_+^2}{\norm{\bx_{|d-k}}^2}\right]+\frac{\norm{\bv}^2}{m}~~~\text{and}~~~
	\norm{\hat{\bw}_{|d-k}}^2~\leq~  \frac{5m}{\min_{i\in[m]}\norm{\bx_{i|d-k}}^2}~,
	\]
	where $\epsilon_{\bv}$ satisfies $\sup_{\bv\in\reals^k}|\epsilon_{\bv}|\leq \epsilon_0$, and.
\end{theorem}
Thus, we see that $\hat{\bw}_{|k}$ is essentially the minimizer of the empirical average of the (weighted) squared hinge loss discussed earlier, plus a certain regularization term which decays with the data size $m$. This is modified by a $(1+\epsilon_{\bv})$ multiplicative parameter, where $|\epsilon_{\bv}|\leq \epsilon_0$ converges uniformly to $0$ as $\norm{E}\rightarrow 0$. As to $\hat{\bw}_{|d-k}$, as in the case of regression, we bound its norm by an expression generally scaling with $m$ but not with $d$, which implies that its contribution to the prediction (assuming a high-dimensional distribution on these coordinates) is negligible as $d\rightarrow \infty$. 

Before continuing, let us informally explain how this squared hinge loss arises in our analysis (with the formal proof deferred as usual to the appendix). To simplify matters, let us suppose that $\{\bx_{i|d-k}\}_{i=1}^{m}$ are precisely orthogonal (so that $E=0$ in the theorem above). In that case, the max-margin predictor $\hat{\bw}$ can be equivalently written as 
\begin{equation}\label{eq:wint}
\arg\min_{\bw\in \reals^d}\norm{\bw_{|k}}^2+\norm{\bw_{|d-k}}^2~~:~~\forall i\in [m],~ y_i\bx_{i|k}^\top\bw_{|k}+y_i\bx_{i|d-k}^\top\bw_{|d-k}\geq 1~.
\end{equation}
For any fixed $\bw_{|k}$, we therefore wish to make $\norm{\bw_{|d-k}}^2$ as small as possible, while satisfying the constraints, which can also be written as $\forall i\in [m],~ y_i\bx_{i|d-k}^\top\bw_{|d-k}\geq 1-y_i\bx_{i|k}^\top\bw_{|k}$. Since $\{y_i\bx_{i|d-k}\}_{i=1}^{m}$ are orthogonal, it is easy to see that we should pick $\hat{\bw}_{|d-k}$ as follows: If $1-y_i\bx_{i|k}^\top\bw_{|k}\leq 0$, we should make $y_i\bx_{i|d-k}^\top\bw_{|d-k}=0$, and if $1-y_i\bx_{i|k}^\top\bw_{|k}> 0$, we should make $y_i\bx_{i|d-k}^\top\bw_{|d-k}= 1-y_i\bx_{i|k}^\top\bw_{|k}$. By orthogonality of the $\{\bx_{i|d-k}\}_{i=1}^{m}$ vectors, it follows that the optimal $\hat{\bw}_{|d-k}$ equals
\[
\sum_{i:1-y_i\bx_{i|k}^\top \bw_{|k}>0}
\left(1-y_i\bx_{i|k}^\top\bw_{|k}\right)\cdot 
\frac{y_i\bx_{i|d-k}}{\norm{\bx_{i|d-k}}^2}~=~
\sum_i \left[1-y_i\bx_{i|k}^\top\bw_{|k}\right]_+\cdot 
\frac{y_i\bx_{i|d-k}}{\norm{\bx_{i|d-k}}^2}~.
\]
Again by orthogonality, it follows that 
$\norm{\bw_{|d-k}}^2=\sum_i \frac{[1-y_i\bx_{i|k}^\top\bw_{|k}]_+^2}{\norm{\bx_{i|d-k}}^2}$.  Plugging this into \eqref{eq:wint}, we get that $\hat{\bw}_{|k}$ equals
\[
\arg\min_{\bw_{|k}\in \reals^k}~ \norm{\bw_{|k}}^2+\sum_{i=1}^{m} \frac{[1-y_i\bx_{i|k}^\top\bw_{|k}]_+^2}{\norm{\bx_{i|d-k}}^2}~~=~~
\arg\min_{\bw_{|k}\in \reals^k}~ \frac{\norm{\bw_{|k}}^2}{m}+\frac{1}{m}\sum_{i=1}^{m} \frac{[1-y_i\bx_{i|k}^\top\bw_{|k}]_+^2}{\norm{\bx_{i|d-k}}^2}~.
\]
Substituting $\bv$ instead of $\bw_{|k}$ results in the expression for $\hat{\bw}_{|k}$ appearing in the theorem (with $\epsilon_{\bv}=0$). The proof of \thmref{thm:classificationdet} essentially generalizes this argument to the case where $\{\bx_{i|d-k}\}_{i=1}^{m}$ are only approximately orthogonal, and also provides a bound for $\norm{\hat{\bw}_{|d-k}}$.

\subsection{Asymptotic Characterization of the Max-margin Predictor}

Having established the perturbation bound in the previous section, let us now show how this can lead to an asymptotic characterization of the max-margin predictor, under suitable distributional assumptions. As in the case of regression, we present a set of sufficient conditions (which are not the most general possible):

\begin{assumption}\label{assump:classification}
	Suppose $\{\Dcal_d\}_{d=k+1}^{\infty}$ is a sequence of distributions on $\reals^d\times \reals$, and $\{m_d\}_{d=k+1}^{\infty}$ a monotonically increasing sequence of positive integers diverging to $\infty$, such that the following hold:
	\begin{enumerate}
	\item \label{assump:lln} Letting $\E_d$ be shorthand for $\E_{(\bx,y)\sim\Dcal_d}$, it holds that $\sup_d \E_d\left[\frac{1+\norm{y\bx_{k}}^4}{\norm{\bx_{|d-k}}^4}\right]<\infty$. 
	\item \label{assump:boundedx} With probability approaching $1$ over sampling $m_d$ samples from $\Dcal_d$, the max-margin predictor $\hat{\bw}_d$ exists, and 
	$
	\max_{i\in [m_d]}\max\left\{\frac{1}{\norm{\bx_{i|d-k}}},\norm{\bx_{i|d-k}}\right\}~\leq~ c
	$ for some constant $c>0$ independent of $d$. 
	\item \label{assump:E} $m_d\cdot\norm{E}\stackrel{P}{\rightarrow} 0$, where $E$ is as defined in \eqref{eq:Edef2} w.r.t. an i.i.d. sample of size $m_d$ from $\Dcal_d$.
	\item \label{assump:Sigma} $m_d\cdot\norm{\E_{d}[\bx_{|d-k}\bx_{|d-k}^\top]}\rightarrow 0$.
	\item \label{assump:boundedmin} There exists some constant $c'>0$ independent of $d$, such that with probability approaching $1$, the function $\hat{g}_d(\bv):= \hat{\E}_d\left[\frac{[1-y\bx_{|k}^\top \bv]_+^2}{\norm{\bx_{|d-k}}^2}\right]=\frac{1}{m_d}\sum_{i=1}^{m_d}\frac{[1-y_i\bx_{i|k}^\top \bv]_+^2}{\norm{\bx_{i|d-k}}^2}$ has a minimizer with norm at most $c'$.
	\item Letting 
	\[
	g_d(\bv):=\E_d\left[\frac{[1-y\bx_{|k}^\top \bv]_+^2}{\norm{\bx_{|d-k}}^2}\right]~,
	\]
	it holds that $\inf_{\bv}\limsup_d \left(g_d(\bv)-\inf_{\bu}g_d(\bu)\right)= 0$.
\end{enumerate}
\end{assumption}

The first four conditions are similar to the condition in Assumption \ref{assump:regression} for regression, with the only differences being that conditions \ref{assump:lln} and \ref{assump:boundedx} require slightly different functions of the data to be bounded. The rest of the assumptions are very mild when we think of the distribution of $y,\bx_{|k},\norm{\bx_{|d-k}}$ as converging to some fixed distribution as $d\rightarrow \infty$ (for instance, in the setting of Example \ref{example:regression}). In that case, we would expect the minimum-norm minimizers of $\hat{\E}_d\left[\frac{[1-y\bx_{|k}^\top \bv]_+^2}{\norm{\bx_{|d-k}}^2}\right]$ to converge to some fixed limit (as $d,m_d\rightarrow \infty$), and therefore condition \ref{assump:boundedmin} should automatically hold. Similarly, condition $6$ should also hold, as it is equivalent to requiring that the set of near-minimizers of the functions $\{g_d(\cdot)\}$ asymptotically overlap. 

With these conditions at hand, we can now state our asymptotic characterization of the max-margin predictor, in terms of the expected (weighted) squared hinge loss function $g_d(\cdot)$ defined above:

\begin{theorem}\label{thm:classification}
	Suppose $\{\Dcal_d\}_{d=k+1}^{\infty}$ and $\{m_d\}_{d=k+1}^{\infty}$ satisfy Assumption \ref{assump:classification}. Then the max-margin predictor $\hat{\bw}_d=\arg\min\norm{\bw}:\forall i\in [m_d],~ y_i\bx_i^\top \bw\geq 1$ satisfies
	\[
	g_d(\hat{\bw}_{d|k})-\inf_{\bv}g_d(\bv)\stackrel{P}{\longrightarrow} 0~~~\text{and}~~~
	\E_{d}\left[(\bx^\top \hat{\bw}_d-\bx_{|k}^\top \hat{\bw}_{d|k})^2\right]\stackrel{P}{\longrightarrow} 0~.
	\]
\end{theorem}

It is interesting to note that the max-margin predictor is asymptotically characterized in terms of a squared hinge loss, even though this loss does not appear explicitly in its definition (and moreover, the max-margin predictor itself arises from training gradient-based methods on losses which are definitely not the squared hinge loss). Instead, the loss naturally arises from our analysis. We note that this loss achieves the same value as the square loss for examples $(\bx,y)$ where $\bx^\top \hat{\bw}_d=y$, but is otherwise distinct. Thus, there is no contradiction with previous results on benign overfitting in classification that focused on situations where the max-margin and min-norm predictors coincide.

\subsection{Implications for Benign Overfitting in Classification}

\thmref{thm:classification} implies that in our binary classification model, the max-margin predictor $\hat{\bw}$ is such that its last $d-k$ coordinates are asymptotically immaterial, whereas the first $k$ coordinates asymptotically minimize the function 
$g_d(\bv)~=~ \E_d\left[\frac{[1-y\bx_{|k}^\top \bv]_+^2}{\norm{\bx_{|d-k}}^2}\right]$. 
Thus, the next natural step is to understand whether the minimizers of this function on the first $k$ coordinates result in good predictors with respect to the misclassification error $\Pr(y\bx^\top\bw\leq 0)$. We note that since this loss is not identical to misclassification error, we cannot hope this to hold in the worst-case: Indeed, see Appendix \ref{app:hingebad} for an explicit example. . However, it is also not an unreasonable requirement, since predictors that attempt to minimize $\frac{[1-y\bx_{|k}^\top \bv]_+^2}{\norm{\bx_{|d-k}}^2}$ over $\bv$ will also tend to make $y\bx_{|k}^\top \bv$ positive, and hence (since the last $d-k$ coordinates have negligible effect) make the expected misclassification error $\Pr(y\bx^\top \bw\leq 0)$ small. A different way to phrase this question is whether the data distribution is such that the weighted squared hinge loss function $\bv\mapsto \frac{[1-y\bx_{|k}^\top \bv]_+^2}{\norm{\bx_{|d-k}}^2}$ is a good \emph{surrogate loss function} for the misclassification error. We note that this question has already been studied for other surrogate losses (see \citet{frei2021agnostic,ji2022agnostic} for some recent examples). 

Our results imply that asymptotically, the max-margin predictor depends only on the joint distribution of $\bx_{|k},\norm{\bx_{|d-k}},y$. Thus, to simplify the discussion in the remainder of this section, we will assume this distribution is fixed for all $d$, and satisfies some mild conditions:
\begin{assumption}\label{assump:classsimple}
	The distribution sequence $\{\Dcal_d\}_{d=k+1}^{\infty}$ on $(\bx,y)\in \reals^d\times \{-1,+1\}$ is such that the joint distribution of $(\bx_{|k},\norm{\bx_{|d-k}},y)$ is the same under any $d$. Moreover, this distribution $\Dcal$ is such that $\E[\bx_{|k}\bx_{|k}^\top]$ is positive definite, and $\Pr(\norm{\bx_{|d-k}}\in [l,u])=1$ for some $[l,u]\subset (0,\infty)$. 
\end{assumption}
To give a simple example, consider the case where $\Dcal_d$ is defined as some fixed distribution over $(\bx_{|k},y)$, and the marginal distribution of $\bx_{|d-k}$ is uniform over some origin-centered sphere\footnote{One can also consider the case of $\bx_{|d-k}$ being a zero-mean Gaussian with covariance matrix  $\frac{\sigma^2}{d-k}I$ for some $\sigma^2>0$, in which case $\norm{\bx_{|d-k}}$ will concentrate around $\sigma$. In the assumption, we slightly simplify this by assuming $\norm{\bx_{|d-k}}$ already has such a limit distribution.}. Under this assumption, the function $g_d$ can be rewritten as 
\[
g(\bv)~:=~\E_{\bx_{|k},z,y	}\left[z\cdot[1-y\bx_{|k}^\top \bv]_+^2\right]~~~,~~~ \bv\in\reals^k
\]
for some fixed distribution over $\bx_{|k}\in \reals^k,y\in\{-1,+1\}$, and $z\in (0,\infty)$.

Our goal now will be to illustrate how our characterization allows us to prove that benign overfitting does occur in some classification setups, which to the best of our knowledge have not been explicitly studied before. 

\begin{figure}[t]
	\centering
	\includegraphics[scale=0.7]{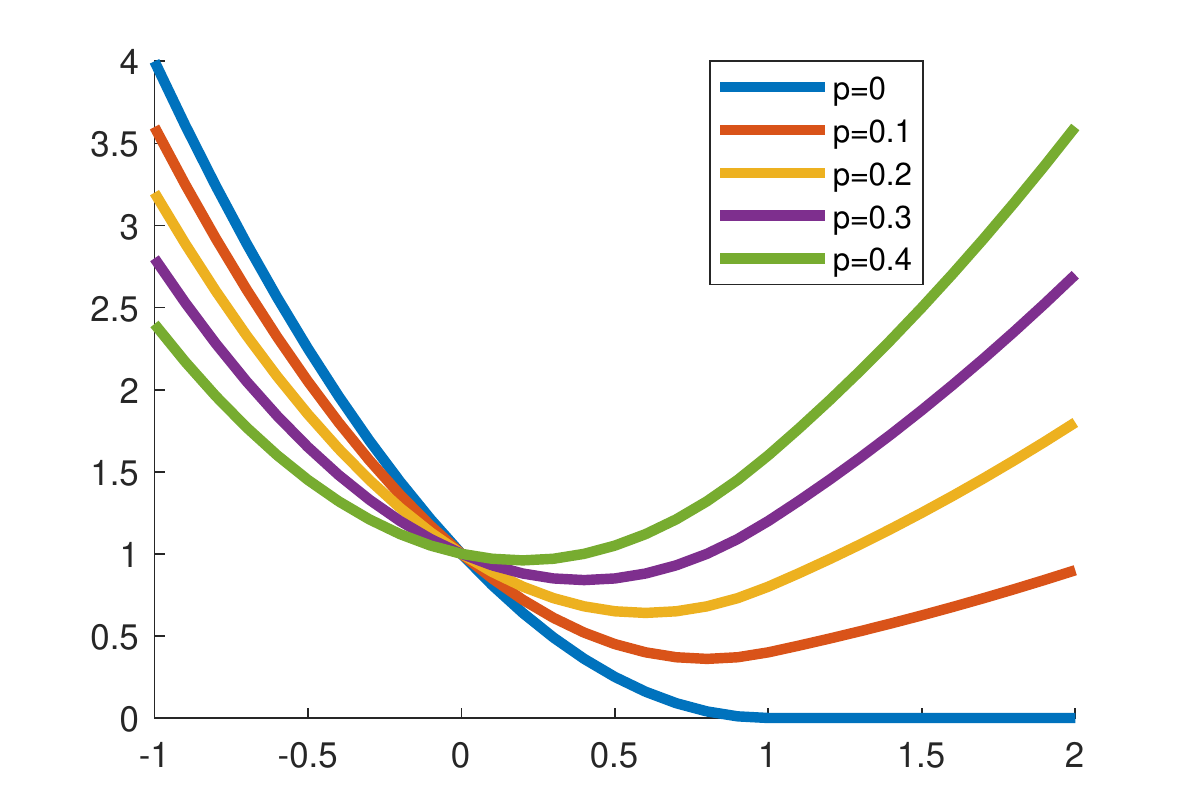}
	\caption{Graphical illustration of the function $\ell_p(\cdot)$ from \eqref{eq:lp} for different values of $p$.}
	\label{fig:lp}
\end{figure}%

For benign overfitting to occur, we need situations where no predictor attains zero error w.r.t. the underlying data distribution. In binary classification setups, the simplest (and most well-studied) case where this occurs is when we have an underlying distribution $\Dcal_{\text{clean}}$ which is \emph{linearly separable} w.r.t. $(\bx_{|k},y)$ (i.e., there is some unit vector $\bw^*\in\reals^k$ such that $\Pr_{\Dcal_{\text{clean}}}(y\bx_{|k}^\top \bw^*< \gamma)=0$ for some margin parameter $\gamma>0$), but where there is random label noise (with each $y$ flipped to $-y$ with some probability $p>0$), resulting in a final distribution $\Dcal$. In such a distribution, $\bw^*$ is still an optimal predictor, but now necessarily its expected misclassification error equals $p$. To model this setting, it will be convenient to assume that $(\bx_{|k},y)$ is still distributed as $\Dcal_{\text{clean}}$, and that we wish to find a predictor $\bw$ satisfying $\Pr_{\Dcal{\text{clean}}}(y\bx_{|k}^\top \bw\leq 0)=0$. However, the predictor $\hat{\bw}$ we learn is with respect to the ``noisy'' labels, and is characterized by the weighted squared hinge loss: Namely, by \thmref{thm:classification}, $\hat{\bw}_{d|k}$ is asymptotically the minimizer of 
\begin{align}
L_p(\bw)~&:=~\E_{(\bx_{|k},z,y)\sim\Dcal_{\text{clean}}}\left[(1-p)\cdot z\cdot[1-y\bx_{|k}^\top \bw]_+^2+p\cdot z\cdot[1+y\bx_{|k}^\top \bw]_+^2\right]\label{eq:lp}\\
&=~\E_{(\bx_{|k},z,y)\sim\Dcal_{\text{clean}}}\E[z\cdot \ell_p(y\bx_{|k}^\top \bw)]~~~\text{where}~~~\ell_p(\beta) := (1-p)\cdot[1-\beta]_+^2+p\cdot[1+\beta]_+^2\notag
\end{align}
(see \figref{fig:lp} for a graphical illustration). It is easily verified that for any $p\in (0,\frac{1}{2})$, $\ell_p$ is a strongly convex function, and therefore $L_p(\cdot)$ is a strongly convex function, as long as $\E[z\cdot \bx_{|k}\bx_{|k}^\top]$ is positive definite  (see \lemref{lem:strongconv} in the appendix for a formal definition of strong convexity and a proof). Therefore, $L_p(\cdot)$  has a unique minimizer $\bw^*_p$, to which $\hat{\bw}_{d|k}$ converges to. Overall, we get that for benign overfitting, it is sufficient that $\bw^*_p$ is an optimal predictor in terms of misclassification error on the ``clean'' labels. This is formalized in the following theorem:  

\begin{theorem}\label{thm:classificationbenign}
	Under the conditions of \thmref{thm:classification} and Assumption \ref{assump:classsimple}, define $R_d(\bw)$ to equal $\Pr_{(\bx,y)\sim\Dcal_d}(y\bx^\top\bw\leq 0)$. Also, suppose that the distribution of $(\bx_{|k},y)$ corresponds to some linearly separable distribution $\Dcal_{\text{clean}}$ with labels flipped with some probability $p\in (0,\frac{1}{2})$.
	
	Then benign overfitting (as defined in \eqref{eq:classificationbenign}) holds
	under the following condition: The (unique) minimizer $\bw_p^*$ of $L_p(\bw)$ satisfies $	\Pr_{(\bx_{|k},y)\sim\Dcal_{\text{clean}}}(y\bx_{|k}^\top\bw_p^*\leq 0)=0$.
\end{theorem}

Focusing on such linearly-separable-with-label-noise distributions, we now turn to study some cases where the condition on $\bw_p^*$ in \thmref{thm:classificationbenign} indeed holds. For example, the following theorem implies that under mild assumptions, just about \emph{any} choice of distribution $\Dcal_{\text{clean}}$ satisfies benign overfitting, for some non-trivial (distribution-dependent) regime of label noise. As far as we can surmise, this is not at all obvious from the original characterization of the max-margin predictor, where the data points appear as constraints and where introducing label noise changes these constraints in possibly complicated ways. However, using our characterization and properties of the squared hinge loss, the result follows from a rather straightforward continuity argument.

\begin{theorem}\label{thm:universalbo}
	Fix any distribution $\Dcal_{\text{clean}}$ over $(\bx_{|k},z,y)$ satisfying Assumption \ref{assump:classsimple}, which is linearly separable w.r.t. $(\bx_{|k},y)$, and where $\bx_{|k}$ has bounded support. Then there exists some $a\in (0,\frac{1}{2})$ (dependent on $\Dcal_{\text{clean}}$), such that for all $p\in (0,a)$, the minimizer $\bw^*_p$ of $L_p(\cdot)$ satisfies $\Pr_{(\bx_{|k},y)\sim\Dcal_{\text{clean}}}(y\bx_{|k}^\top \bw^*_p\leq 0)=0$. 
\end{theorem}
Intuitively, the proof proceeds by arguing that $\bw^*_p$ is continuous in $p\in (0,\frac{1}{2})$, and as $p\rightarrow 0$, necessarily converges to some predictor which separates the data with positive margin. Hence, small perturbations of the predictor will maintain linear separability, and therefore $\bw^*_p$ will remain a linear separator for small positive values of $p$. 

This result holds for generic linearly separable distributions, but does not specify the amount of label noise under which benign overfitting occurs. In the following theorem, we identify one simple class of distributions where benign overfitting occurs with any amount of label noise up to $\frac{1}{2}$:
\begin{theorem}\label{thm:p12}
		Fix any distribution $\Dcal_{\text{clean}}$ on $(\bx_{|k},z,y)$ satisfying Assumption \ref{assump:classsimple}, such that the distribution of $(\bx_{|k},y)$ is linearly separable. Moreover, suppose that for some unit vector $\bu$, and conditioned on any value of $y$, $\bu^\top \bx$ and $(I-\bu\bu^\top)\bx$ are mutually independent, and the distributions of $(I-\bu\bu^\top)\bx$ and $-(I-\bu\bu^\top)\bx$ are identical. Then for all $p\in (0,\frac{1}{2})$, the minimizer $\bw^*_p$ of $L_p(\cdot)$ satisfies $\Pr_{\Dcal_{\text{clean}}}(y\bx_{|k}^\top \bw^*_p\leq 0)=0$.
\end{theorem}
The conditions in the theorem refer to a situation where there is some distinguished direction $\bu$, such that conditioned on $y$, $\bx_{|k}=\beta\bu+\bs$ for some mutually independent random variables $\beta,\bs$ where $\bs$ is orthogonal to $\bu$ and has a symmetric distribution. Examples where this occurs include any one-dimensional distribution, and a mixture of any two symmetric distributions with means in $\text{span}(\bu)$ (one for $y=1$ and one for $y=-1$, and assuming linear separability). Note that unlike most previous results on benign overfitting in classification, the distributions do not need to be identical nor satisfy any additional structural properties.

\section{Discussion}\label{sec:discussion}

In this paper, we presented several new results on benign overfitting, for both regression and classification. For linear regression with the square loss, we argued that benign overfitting should not be expected to hold in general, once we go beyond well-specified distributions. Moreover, we showed how this can be extended beyond linear regression with the square loss, by an argument proving how the existence of benign overfitting on some regression problems precludes its existence on other regression problems. On the more positive side, for classification problems, we showed that the max-margin is implicitly biased towards minimizing a weighted squared hinge loss w.r.t. the underlying distribution (at least in a model where an arbitrary $k$-dimensional distribution is concatenated with a high-dimensional distribution). We use it to show benign overfitting in various settings, by considering cases where this  squared hinge loss is a good surrogate for the misclassification error. 

Overall, we hope that our observations here will allow us to understand benign overfitting beyond the settings studied so far in the literature. For example, it would be interesting to identify other settings where the structure of the squared hinge loss means that the max-margin predictor will have benign overfitting properties. Moreover, our results focused on input distributions with a clean separation between a few ``important'' coordinates, and a high dimensional distribution on the other coordinates. Although this is a prototypical setting for benign overfitting, our insights can potentially be extended to other input distributions, and identifying them can be an interesting direction for future research. 

Another, more technical issue is that our asymptotic characterization of the min-norm or max-margin predictor require the high-dimensional distribution to be sufficiently ``spread'', which ultimately requires the dimension $d$ to scale sufficiently faster than the sample size $m_d$. For instance, as discussed in Example \ref{example:regression}, if we consider a distribution on $\bx\in\reals^d$ such that $\bx_{|d-k}$ is zero-mean Gaussian with covariance matrix $\frac{1}{d-k}\cdot I$, then we need $m_d^3\log(d)/d\rightarrow 0$ for the perturbation matrix $E$ to decay sufficiently fast, and for \thmref{thm:regression} to hold. A similar requirement also applies to our classification results. For the purposes of our paper, this is not a major issue, since our goal was to understand when benign overfitting might or might not occur assuming the dimension is sufficiently large, and indeed we did not attempt to optimize this condition. Nevertheless, since other papers usually assume a milder scaling of $d$ vs. $m_d$, it would be interesting to understand whether results similar to ours can be obtained under such conditions.

\subsection*{Acknowledgements}
This research is supported in part by European Research Council (ERC) grant 754705. We thank Gilad Yehudai and the anonymous JMLR reviewers for several very helpful comments and suggestions.

\bibliographystyle{plainnat}
\bibliography{bib}

 \appendix

\section{Proofs}\label{app:proofs}

\subsection{Proof of \thmref{thm:regressiondet}}

We will utilize the following matrix inverse perturbation result:
\begin{lemma}\label{lem:pertinv}
	Let $A$ be some positive definite matrix with minimal eigenvalue $\lambda_{\min}(A)>0$. Then for any symmetric matrix $E$ of the same size such that $\norm{E}\leq \frac{\lambda_{\min}(A)}{2}$, it holds that $A+E$ is invertible and
	\[
	\norm{(A+E)^{-1}-A^{-1}}\leq \frac{2}{\lambda_{\min}(A)^2}\norm{E}.
	\]
\end{lemma}
\begin{proof}
	By Weyl's inequality, $\lambda_{\min}(A+E)\geq \lambda_{\min}(A)-\norm{E}\geq \frac{\lambda_{\min}(A)}{2}>0$, hence $A+E$ is invertible. Moreover, by the Woodbury matrix identity,
	\[
	(A+E)^{-1}~=~
	A^{-1}-A^{-1}E(I+A^{-1}E)^{-1}A^{-1}~,
	\]
	which implies
	\[
	\norm{(A+E)^{-1}-A^{-1}}~\leq~ \norm{A^{-1}}^2\cdot \norm{E}\cdot \norm{(I+A^{-1}E)^{-1}}~\leq~ \frac{\norm{A^{-1}}^2\cdot \norm{E}}{1-\norm{A^{-1}}\cdot \norm{E}}~.
	\]
	Noting that $\norm{A^{-1}}= \frac{1}{\lambda_{\min}(A)}$ and $\norm{E}\leq \frac{\lambda_{\min}(A)}{2}$, it follows that the above is at most $\frac{\frac{1}{\lambda_{\min}(A)^2}\norm{E}}{1-\frac{1}{\lambda_{\min}(A)}\cdot \frac{\lambda_{\min}(A)}{2}} = \frac{2}{\lambda_{\min}(A)^2}\norm{E}$.
\end{proof}

We now turn to the proof itself. Let $X\in \reals^{m\times d}$ be the matrix whose $i$-th row is $\bx_i$, and let $X_{|k},X_{|d-k}$ be its first $k$ and last $d-k$ columns respectively (so that the $i$-th row of $X_{|k}$ is $\bx_{i|k}$, and that of $X_{|d-k}$ is $\bx_{i|d-k}$). Also, let $\by=(y_1,\ldots,y_m)$. Since $\{\bx_i\}_{i=1}^{m}$ are linearly independent, $X$ has full row rank. Therefore, $XX^\top$ is invertible, and $\hat{\bw}=\arg\min \norm{\bw}:X\bw=\by$ can be written in closed form as
\[
\hat{\bw}~=~ X^\top(XX^\top)^{-1}\by~=~
X^\top\left(X_{|k} X_{|k}^\top + X_{|d-k} X_{|d-k}^\top\right)^{-1}\by~=~
X^\top\left(D+X_{|k} X_{|k}^\top +E\right)^{-1}\by~,
\]
where $E$ is as defined in the theorem statement (namely, the off-diagonal entries of $X_{|d-k}X_{|d-k}^\top$), and 
\[
D~:=~ X_{|d-k}X_{|d-k}^\top-E~=~ \text{diag}(\norm{\bx_{i|d-k}}^2,\ldots,\norm{\bx_{m|d-k}}^2)
\]
is a diagonal matrix. In what follows, it will be useful to note that $\lambda_{\min}(D)=\min_{i\in [m]}\norm{\bx_{m|d-k}}^2$.
	
Continuing, we have by the representation of $\hat{\bw}$ above that
\begin{equation}\label{eq:wdfirst}
	\hat{\bw}~=~X^\top (D+X_{|k}X_{|k}^\top)^{-1}\by + X^\top E' \by~~~\text{where}~~~E' = (D+X_{|k}X_{|k}^\top+E)^{-1}-(D+X_{|k}X_{|k}^\top)^{-1}~.
\end{equation}
Considering the first $k$ and last $d-k$ coordinates separately, it follows that
\begin{equation}\label{eq:wdfirst0}
	\hat{\bw}_{|k}~=~X_{|k}^\top (D+X_{|k}X_{|k}^\top)^{-1}\by + X_{|k}^\top E' \by~~~\text{and}~~~\hat{\bw}_{|d-k}~=~X_{|d-k}^\top (D+X_{|k}X_{|k}^\top)^{-1}\by + X_{|d-k}^\top E' \by~,
\end{equation}
where $E'$ is as defined in \eqref{eq:wdfirst}. 

Applying \lemref{lem:pertinv} with $A=D+X_{|k}X_{|k}^\top$ (noting that $\lambda_{\min}(D+X_{|k}X_{|k}^\top)\geq \lambda_{\min}(D)=\min_{i\in[m]}\norm{\bx_{i|d-k}}^2>0$ and that we assume $\norm{E}\leq \frac{\lambda_{\min}(D)}{2}$), it follows that
$\norm{E'}\leq \frac{2}{\lambda_{\min}(D)^2}\norm{E}$. As a result, and using the fact that the spectral norm is upper bounded by the Frobenius norm, we have
\begin{align*}
	\norm{X_{|k}^\top E' \by}~&\leq~\norm{X_{|k}}\cdot \norm{\by}\cdot \norm{E'}~\leq~ m\cdot \sqrt{\frac{1}{m}\norm{X_{|k}}_F^2}\cdot  \sqrt{\frac{1}{m}\norm{\by}^2}\cdot \frac{2}{\lambda_{\min}(D)^2}\norm{E}\\
	&=~ \frac{2m\norm{E}}{\lambda_{\min}(D)^2}\cdot\sqrt{\hat{\E}[\norm{\bx_{|k}}^2]\cdot \hat{\E}[y^2]}~.
\end{align*}
An identical calculation implies that $\norm{X_{|d-k}^\top E' \by}\leq \frac{2m\norm{E}}{\lambda_{\min}(D)^2}\cdot\sqrt{\hat{\E}[\norm{\bx_{|d-k}}^2]\cdot \hat{\E}[y^2]}$. Plugging these bounds back into \eqref{eq:wdfirst0}, it follows that
\begin{equation}\label{eq:wdsecond11}
	\left\|\hat{\bw}_{|k}-X_{|k}^\top (D+X_{|k}X_{|k}^\top)^{-1}\by\right\|~\leq~ \frac{2m\norm{E}}{\lambda_{\min}(D)^2}\cdot\sqrt{\hat{\E}[\norm{\bx_{|k}}^2]\cdot \hat{\E}[y^2]}
\end{equation}
and
\begin{equation}\label{eq:wdsecond12}
	\left\|\hat{\bw}_{|d-k}-X_{|d-k}^\top (D+X_{|k}X_{|k}^\top)^{-1}\by\right\|~\leq~ \frac{2m\norm{E}}{\lambda_{\min}(D)^2}\cdot\sqrt{\hat{\E}[\norm{\bx_{|d-k}}^2]\cdot \hat{\E}[y^2]}~.
\end{equation}

We now turn to analyze $X_{|k}^\top (D+X_{|k}X_{|k}^\top)^{-1}\by$ and $X_{|d-k}^\top (D+X_{|k}X_{|k}^\top)^{-1}\by$, starting with the first expression. Using the Woodbury matrix identity, we have
\begin{align*}
	(D+X_{|k} X_{|k}^\top)^{-1} ~&=~ D^{-1}- D^{-1}X_{|k}(I+X_{|k}^\top D^{-1} X_{|k})^{-1}X_{|k}^\top D^{-1}\\
	&=~ D^{-1}\left(I-X_{|k}(I+X_{|k}^\top D^{-1} X_{|k})^{-1}X_{|k}^\top D^{-1}\right)
\end{align*}
(note that $I+X_{|k}^\top D^{-1} X_{|k}$ is indeed invertible, since it is the sum of the identity matrix and a positive semidefinite matrix). This implies that
\begin{align}
	X_{|k}^\top\left(D+X_{|k} X_{|k}^\top\right)^{-1}\by~&=~ 
	X_{|k}^\top D^{-1}\left(I-X_{|k}(I+X_{|k}^\top D^{-1} X_{|k})^{-1}X_{|k}^\top D^{-1}\right)\by\notag\\
	&=~\left(I-X_{|k}^\top D^{-1}X_{|k}\left(I+X_{|k}^\top D^{-1} X_{|k}\right)^{-1}\right)X_{|k}^\top D^{-1}\by\notag\\
	&=~\left(\left(I+X_{|k}^\top D^{-1} X_{|k}\right)-X_{|k}^\top D^{-1}X_{|k}\right)\left(I+X_{|k}^\top D^{-1} X_{|k}\right)^{-1}X_{|k}^\top D^{-1}\by\notag\\
	&=~\left(I+X_{|k}^\top D^{-1} X_{|k}\right)^{-1}X_{|k}^\top D^{-1}\by\notag\\
	&=~\left(\frac{1}{m}I+\frac{1}{m}X_{|k}^\top D^{-1} X_{|k}\right)^{-1}\left(\frac{1}{m} X_{|k}^\top D^{-1}\by\right)\notag\\
	&=~\left(\frac{1}{m}I +\hat{\E}\left[\frac{\bx_{|k}\bx_{|k}^\top}{\norm{\bx_{|d-k}}^2}\right]\right)^{-1}\cdot \hat{\E}\left[\frac{y\bx_{|k}}{\norm{\bx_{|d-k}}^2}\right]~.\label{eq:wddd}
\end{align}
Using \lemref{lem:pertinv} and the assumption that $\frac{1}{m}\leq \frac{1}{2}\lambda_{\min}\left(\hat{\E}\left[\frac{\bx_{|k}\bx_{|k}^\top}{\norm{\bx_{|d-k}}^2}\right]\right)$, it follows that
\[
\left\|\left(\frac{1}{m}I +\hat{\E}\left[\frac{\bx_{|k}\bx_{|k}^\top}{\norm{\bx_{|d-k}}^2}\right]\right)^{-1}-\left(\hat{\E}\left[\frac{\bx_{|k}\bx_{|k}^\top}{\norm{\bx_{|d-k}}^2}\right]\right)^{-1}\right\|~\leq~ \frac{2}{ \lambda_{\min}\left(\hat{\E}\left[\frac{\bx_{|k}\bx_{|k}^\top}{\norm{\bx_{|d-k}}^2}\right]\right)^2\cdot m}~.
\]
Combining this with \eqref{eq:wddd} and the Cauchy-Schwarz inequality, it follows that
\begin{align*}
	&\left\|X_{|k}^\top\left(D+X_{|k} X_{|k}^\top\right)^{-1}\by-\left(\hat{\E}\left[\frac{\bx_{|k}\bx_{|k}^\top}{\norm{\bx_{|d-k}}^2}\right]\right)^{-1}\cdot \hat{\E}\left[\frac{y\bx_{|k}}{\norm{\bx_{|d-k}}^2}\right]\right\|\\
	&=~ \left\|\left(\left(\frac{1}{m}I +\hat{\E}\left[\frac{\bx_{|k}\bx_{|k}^\top}{\norm{\bx_{|d-k}}^2}\right]\right)^{-1}-\left(\hat{\E}\left[\frac{\bx_{|k}\bx_{|k}^\top}{\norm{\bx_{|d-k}}^2}\right]\right)^{-1}\right)\cdot\hat{\E}\left[\frac{y\bx_{|k}}{\norm{\bx_{|d-k}}^2}\right]\right\|\\	
	&\leq~ \left\|\left(\frac{1}{m}I +\hat{\E}\left[\frac{\bx_{|k}\bx_{|k}^\top}{\norm{\bx_{|d-k}}^2}\right]\right)^{-1}-\left(\hat{\E}\left[\frac{\bx_{|k}\bx_{|k}^\top}{\norm{\bx_{|d-k}}^2}\right]\right)^{-1}\right\|\cdot\left\|\hat{\E}\left[\frac{y\bx_{|k}}{\norm{\bx_{|d-k}}^2}\right]\right\|\\
	&\leq~ \frac{2}{ \lambda_{\min}\left(\hat{\E}\left[\frac{\bx_{|k}\bx_{|k}^\top}{\norm{\bx_{|d-k}}^2}\right]\right)^2\cdot m}\cdot \left\|\hat{\E}\left[\frac{y\bx_{|k}}{\norm{\bx_{|d-k}}^2}\right]\right\|~.
\end{align*}
Combining the above with \eqref{eq:wdsecond11} using a triangle inequality, it follows that
\begin{align*}
	&\left\|\hat{\bw}_{|k}-\left(\hat{\E}\left[\frac{\bx_{|k}\bx_{|k}^\top}{\norm{\bx_{|d-k}}^2}\right]\right)^{-1}\cdot \hat{\E}\left[\frac{y\bx_{|k}}{\norm{\bx_{|d-k}}^2}\right]\right\|\\
	&~~~~~~~\leq~\frac{2\left\|\hat{\E}\left[\frac{y\bx_{|k}}{\norm{\bx_{|d-k}}^2}\right]\right\|}{\lambda_{\min}\left(\hat{\E}\left[\frac{\bx_{|k}\bx_{|k}^\top}{\norm{\bx_{|d-k}}^2}\right]\right)^2\cdot m}+\frac{2m\norm{E}}{\lambda_{\min}(D)^2}\cdot\sqrt{\hat{\E}[\norm{\bx_{|k}}^2]\cdot \hat{\E}[y^2]}~.
\end{align*}
Recalling that $\lambda_{\min}(D)=\min_{i\in [m]}\norm{\bx_{i|d-k}}^2$, the first bound in the theorem follows.

As to bound on $\norm{\bw_{|d-k}}$ in the theorem, recalling \eqref{eq:wdsecond12}, we need to analyze the expression $X_{|d-k}^\top (D+X_{|k}X_{|k}^\top)^{-1}\by$. By Cauchy-Schwarz, its norm is at most
\begin{align*}
\norm{X_{|d-k}}\cdot \left\|(D+X_{|k}X_{|k}^\top)^{-1}\right\|\cdot \norm{\by}~&=~ \frac{\norm{X_{|d-k}}\cdot \norm{\by}}{\lambda_{\min}(D+X_{|k}X_{|k}^\top)}~\leq~\frac{\norm{X_{|d-k}}_F\cdot \norm{\by}}{\lambda_{\min}(D)}\\
&=~ \frac{m\cdot \sqrt{\hat{\E}[\norm{\bx_{|d-k}}^2]\cdot \hat{\E}[y^2]}}{\lambda_{\min}(D)}~.
\end{align*}
Combining this with \eqref{eq:wdsecond12}, it follows that
\[
\norm{\hat{\bw}_{|d-k}}~\leq~ \frac{m\cdot \sqrt{\hat{\E}[\norm{\bx_{|d-k}}^2]\cdot \hat{\E}[y^2]}}{\lambda_{\min}(D)}\cdot\left(1+\frac{2\norm{E}}{\lambda_{\min}(D)}\right)~.
\]
Recalling that $\lambda_{\min}(D)=\min_{i\in [m]}\norm{\bx_{i|d-k}}^2$, the second bound in the theorem follows.

\subsection{Proof of Theorem \ref{thm:regression}}

Assumption $1$ implies that the law of large numbers holds with respect to the random variables $\norm{\bx_{|k}}^2$, $\norm{\bx_{|d-k}}^2$, $\frac{y\bx_{|k}}{\norm{\bx_{|d-k}}^2}$ and $\frac{\bx_{|k}\bx_{|k}^\top}{\norm{\bx_{|d-k}}^2}$, so their empirical average over $m_d$ i.i.d. instances (where $m_d\rightarrow \infty$) converges in probability to their (finite) expectations. Moreover, by assumption $2$, $\min_{i\in [m_d]}\cdot\norm{\bx_{i|d-k}}$ is at least some positive constant with probability approaching $1$, and $m_d\norm{E}$ becomes arbitrarily small by assumption $3$. All this implies that as $d,m_d\rightarrow \infty$, then with probability approaching $1$, the conditions of \thmref{thm:regressiondet} hold, and the upper bound in its first displayed equation is arbitrarily small. Hence, $\left\|\hat{\bw}_{d|k}-\left(\hat{\E}_d\left[\frac{\bx_{|k}\bx_{|k}^\top}{\norm{\bx_{|d-k}}^2}\right]\right)^{-1} \hat{\E}_d\left[\frac{y\bx_{|k}}{\norm{\bx_{|d-k}}^2}\right]\right\|\stackrel{P}{\rightarrow} 0$, where $\hat{\E}_d[\cdot]$ is the uniform distribution over $\{\bx_i,y_i\}_{i=1}^{m_d}$ sampled i.i.d. from $\Dcal_d$. To prove the first part of the theorem, it remains to show that the expression $\left(\hat{\E}_d\left[\frac{\bx_{|k}\bx_{|k}^\top}{\norm{\bx_{|d-k}}^2}\right]\right)^{-1} \hat{\E}_d\left[\frac{y\bx_{|k}}{\norm{\bx_{|d-k}}^2}\right]$ converges in probability to 
$\left(\E_d\left[\frac{\bx_{|k}\bx_{|k}^\top}{\norm{\bx_{|d-k}}^2}\right]\right)^{-1} \E_d\left[\frac{y\bx_{|k}}{\norm{\bx_{|d-k}}^2}\right]$. This holds, because again by the law of large numbers and our assumptions,
\[
\left\|\hat{\E}_d\left[\frac{y\bx_{|k}}{\norm{\bx_{|d-k}}^2}\right]-\E_d\left[\frac{y\bx_{|k}}{\norm{\bx_{|d-k}}^2}\right]\right\|\stackrel{P}{\longrightarrow} 0
~~\text{and}~~ \left\|\hat{\E}_d\left[\frac{\bx_{|k}\bx_{|k}^\top}{\norm{\bx_{|d-k}}^2}\right]-\E_d\left[\frac{\bx_{|k}\bx_{|k}^\top}{\norm{\bx_{|d-k}}^2}\right]\right\|\stackrel{P}{\rightarrow}~0,
\]
and moreover $\E_d\left[\frac{\bx_{|k}\bx_{|k}^\top}{\norm{\bx_{|d-k}}^2}\right]$  is positive definite by assumption $1$, hence the convergence also holds with respect to the inverse of the matrices.

As to the second part, recalling that $\E_d$ is shorthand for $\E_{(\bx,y)\sim \Dcal_d}$, and fixing some $\hat{\bw}_d$, we have 
\begin{align*}
	\E_d\left[\left(\bx^\top\hat{\bw}_d-\hat{\bx}_{|k}^\top\hat{\bw}_{d|k}\right)^2\right]
	&=~
	\E_d\left[\left(\bx_{|d-k}^\top\hat{\bw}_{d|d-k}\right)^2\right]~=~ \hat{\bw}_{d|d-k}^\top \E_d[\bx_{d-k}\bx_{|d-k}^\top]\hat{\bw}_{d|d-k}\\
	&\leq~ \norm{\hat{\bw}_{d|d-k}}^2\cdot\norm{\E_d[\bx_{d-k}\bx_{|d-k}^\top]}~.
\end{align*}
Assumptions $1,2,3$ and \thmref{thm:regressiondet} imply that with probability approaching $1$, $\norm{\hat{\bw}_{d|d-k}}\leq c'\cdot m_d$, where $c'>0$ is a constant independent of $d$. Plugging into the displayed equation above, and using assumption $4$, implies that the bound in the displayed equation converges to $0$. This holds with probability approaching $1$ over the choice of $\hat{\bw}_d$, hence overall $\E_d\left[\left(\bx^\top\hat{\bw}_d-\hat{\bx}_{|k}^\top\hat{\bw}_{d|k}\right)^2\right]\stackrel{P}{\rightarrow} 0$.

\subsection{Proof of \lemref{lem:linear}}
	Considering $\xi$ which equals $0$ almost surely, we clearly have $\sigma^{-1}(0)=0$. More generally, fix some $c\in \reals$ and $z>0$, and consider the random variable
	\[
	\xi~=~\begin{cases} -c& \text{w.p.}~\frac{z}{z+1}\\
		cz& \text{w.p.}~ \frac{1}{z+1}~,
	\end{cases}
	\]
	which is easily verified to be zero mean. The assumption $\E[\sigma^{-1}(\xi)]=0$ translates to
	\[
	\frac{z}{z+1}\cdot \sigma^{-1}(-c)+\frac{1}{z+1}\cdot \sigma^{-1}(cz)=0~~~\Longrightarrow~~~
	\sigma^{-1}(cz) = -\sigma^{-1}(-c)\cdot z~.
	\]
	Fixing $c=1$ and studying this equation as a function of $z>0$, we see that $\sigma^{-1}(\cdot)$ is necessarily linear over $[0,\infty)$ (with slope $-\sigma^{-1}(-1)$). Similarly, fixing $c=-1$, we get that $\sigma^{-1}$ is necessarily linear over $(-\infty,0]$ (with slope $\sigma^{-1}(1)$). Thus, it only remains to show that $-\sigma^{-1}(-1)=\sigma^{-1}(1)$, which follows by considering the random variable $\xi$ uniformly distributed on $\{-1,1\}$, and noting that $\E[\sigma^{-1}(\xi)]=0$ implies $\sigma^{-1}(-1)+\sigma^{-1}(1)=0$ in this case.

\subsection{Proof of \thmref{thm:classificationdet}}

Consider the optimization problem
\[
\min_{\balpha\in \reals^m} \balpha^\top D\balpha~:~ D\balpha\succeq \br~~~
\Longleftrightarrow~~~ \min_{\balpha\in \reals^{m}}\sum_{i=1}^{m}D_{i,i}\alpha_i^2~:~\forall i\in [m],~ D_{i,i}\alpha_i\geq r_i~
\]
for some diagonal matrix $D\in\reals^{m\times m}$ with positive entries $D_{i,i}$ on the diagonal, and a vector $\br\in\reals^m$. Note that for this problem, it is easily verified that the optimum satisfies $\alpha_i = \frac{[r_i]_+}{D_{i,i}}$ for all $i\in [m]$, hence the optimal value is $\norm{[D^{-1/2}\br]_+}^2=\sum_{i=1}^{d}\frac{[r_i]_+^2}{D_{i,i}}$. The following key technical lemma quantifies by how much the optimal value changes, if we perturb $D$ be some symmetric matrix $E$:
\begin{lemma}\label{lem:almostorth}
	Fix some integer $m\geq 1$, some $\br\in \reals^m$, a diagonal matrix $D\in \reals^{m\times m}$ with positive diagonal entries, and an $m\times m$ symmetric matrix $E$ such that $\norm{E}<\frac{\lambda_{\min}(D)}{2}$. Then the set $\{\balpha\in \reals^m:(D+E)\balpha\succeq \br\}$ is not empty. Moreover, letting
	\begin{equation}\label{eq:lema}
		a^*~=~\min_{\balpha\in \reals^m} \balpha^\top (D+E)\balpha~~:~~ (D+E)\balpha \succeq \br~,
	\end{equation}
	we have
	\[
	\left|a^*-\norm{[D^{-1/2}\br]_+}^2\right|~\leq~ \frac{2\norm{E}\cdot\norm{D}}{\lambda_{\min}(D)^2}\cdot \norm{[D^{-1/2}\br]_+}^2~.
	\]
\end{lemma}
\begin{proof}
	The fact that $\{\balpha\in \reals^m:(D+E)\balpha\succeq \br\}$ is not empty follows from the observation that $D+E$ is positive definite and hence invertible (since we assume $\norm{E}\leq \frac{\lambda_{\min}(D)}{2}$). Thus, the set contains for instance the vector $(D+E)^{-1}\br$. Also, note that the minimum $a^*$ is indeed attained, as we are minimizing a strongly convex function with (feasible) linear constraints. 
	
	To continue, let us perform the variable change $\bbeta=(D+E)\balpha$ (which is valid since $D+E$ is invertible), so $\balpha=(D+E)^{-1}\bbeta$, and 
	\begin{equation}\label{eq:astar}
		a^*~=~ \min_{\bbeta\in\reals^m} \bbeta^\top(D+E)^{-1}\bbeta~:~\bbeta\succeq \br~.
	\end{equation}
	By \lemref{lem:pertinv} and the assumption $\norm{E}\leq \frac{\lambda_{\min}(D)}{2}$, it follows that
	\[
	\norm{(D+E)^{-1}-D^{-1}}~\leq~ \frac{2\norm{E}}{\lambda_{\min}(D)^2}~.
	\]
	This implies that 
	\[
	D^{-1}+\frac{2\norm{E}}{\lambda_{\min}(D)^2}I~\succeq~ (D+E)^{-1}~\succeq~
	D^{-1}-\frac{2\norm{E}}{\lambda_{\min}(D)^2}I,
	\]
	where $A\succeq B$ for symmetric matrices $A,B$ implies that $A-B$ is positive semidefinite. Plugging this back into \eqref{eq:astar}, it follows that
	\[
	\min_{\bbeta\in\reals^m:\bbeta\succeq \br}\bbeta^\top\left(D^{-1}+\frac{2\norm{E}}{\lambda_{\min}(D)^2}I\right)\bbeta~\geq~ a^*~\geq~\min_{\bbeta\in\reals^m:\bbeta\succeq \br}\bbeta^\top\left(D^{-1}-\frac{2\norm{E}}{\lambda_{\min}(D)^2}I\right)\bbeta~.
	\]
	Now, it is easily verified that for a diagonal matrix $A\in\reals^{m\times m}$ with non-negative diagonal entries,
	\[
	\min_{\bbeta:\bbeta\succeq \br}\bbeta^\top A \bbeta = \min_{\bbeta:\forall i,\beta_i\geq r_i}\sum_{i=1}^{m} A_{i,i} \beta_i^2~=~\sum_{i=1}^{m} A_{i,i} [r_i]_+^2.
	\]
	Plugging this into the previous displayed equation, it follows that
	\[
	\sum_{i=1}^{m} \left(\frac{1}{D_{i,i}}+\frac{2\norm{E}}{\lambda_{\min}(D)^2}\right)[r_i]_+^2~\geq~
	a^*~\geq~
	\sum_{i=1}^{m} \left(\frac{1}{D_{i,i}}-\frac{2\norm{E}}{\lambda_{\min}(D)^2}\right)[r_i]_+^2~.
	\]
	Therefore,
	\[
	\left|a^*-\sum_{i=1}^{m}\frac{[r_i]_+^2}{D_{i,i}}\right|~\leq~
	\sum_{i=1}^{m}[r_i]_+^2\cdot \frac{2\norm{E}}{\lambda_{\min}(D)^2}
	~=~
	\sum_{i=1}^{m}\frac{[r_i]_+^2}{D_{i,i}}\cdot \frac{2\norm{E}D_{i,i}}{\lambda_{\min}(D)^2}
	~\leq~
	\sum_{i=1}^{m}\frac{[r_i]_+^2}{D_{i,i}}\cdot \frac{2\norm{E}\cdot \norm{D}}{\lambda_{\min}(D)^2}~.
	\]
	Noting that $\sum_{i=1}^{m}\frac{[r_i]_+^2}{D_{i,i}} = \norm{[D^{-1/2}\br]_+}^2$ and plugging in the displayed equation above, the bound in the lemma follows. 
\end{proof}

With this lemma in hand, we can now turn to prove the theorem. $\hat{\bw}$ can be equivalently written as 
\[
\hat{\bw}~=~ \arg\min_{\bw\in \reals^d}~\frac{\norm{\bw}^2}{m}~~:~~\forall i\in [m],~ y_i \bx_i^\top \bw\geq 1~.
\]
writing $\bw=(\bv,\bu)$ where $\bv\in\reals^k,\bu\in\reals^{d-k}$, the above is equivalent to
\begin{equation}\label{eq:optprob00}
\arg\min_{\bv\in \reals^k,\bu\in \reals^{d-k}} \frac{\norm{\bv}^2}{m}+\frac{\norm{\bu}^2}{m}~~:~~\forall i\in [m],~ y_i \bx_{i|d-k}^\top\bu \geq 1-y_i\bx_{i|k}^\top \bv~.
\end{equation}
This in turn is equivalent to
\begin{equation}\label{eq:optprob0}
	\arg\min_{\bv\in \reals^k} \frac{\norm{\bv}^2}{m}+f_{m}(\bv)
\end{equation}
where
\[
f_{m}(\bv)~=~ \min_{\bu\in \reals^{d-k}} \frac{\norm{\bu^2}}{m}~~:~~\forall i\in [m],~y_i\bx_{i|d-k}^\top\bu \geq 1-y_i\bx_{i|k}^\top \bv~.
\]
Let $Z_{|k},Z_{|d-k}$ be $m\times (d-k)$ matrices whose $i$-th rows are (respectively) $y_i\bx_{i}$ and $y_i\bx_{i|d-k}$. Also, let $\mathbf{1}$ be the all-ones vector in $\reals^{m}$. Thus, we can write
\[
f_{m}(\bv)~=~\min_{\bu\in\reals^{d-k}}\frac{\norm{\bu}^2}{m}~~:~~ Z_{|d-k}\bu\succeq \mathbf{1}-Z_{|k}\bv~.
\]
Clearly, the optimal $\bu$ must lie in the row span of $Z_{|d-k}$ (otherwise, we can further reduce $\norm{\bu}^2$ by projecting to that subspace, without violating the constraints). Thus, any optimal $\bu$ can be written as $Z_{|d-k}^\top\balpha$ for some $\balpha\in \reals^{m}$, so we can rewrite the displayed equation above as
\[
f_{m}(\bv)~=~\min_{\balpha\in \reals^{m}}\frac{1}{m}\balpha^\top\left(Z_{|d-k} Z_{|d-k}^\top\right)\balpha~~:~~Z_{|d-k}Z_{|d-k}^\top\balpha\succeq\mathbf{1}-Z_{|k}\bv~.
\]
Letting
\[
Z_{|d-k}Z_{|d-k}^\top~=~D+E~,
\]
where $E$ is as defined in the theorem statement, and $D$ being a diagonal matrix consisting of the diagonal of $Z_{|d-k}Z_{|d-k}^\top$ (namely $(\norm{\bx_{1|d-k}}^2,\ldots,\norm{\bx_{m|d-k}}^2)$), we can write the above as
\[
	f_{m}(\bv)~=~\min_{\balpha\in \reals^{m}}\frac{1}{m}\balpha^\top(D+E)\balpha~~:~~(D+E)\balpha \succeq\mathbf{1}-Z_{|k}\bv~.
\]
Applying \lemref{lem:almostorth} on the equation above (and noting that $\norm{E}< \frac{\lambda_{\min}(D)}{2}$, which follows from the assumption $\frac{2\norm{E}\cdot\norm{D}}{\lambda_{\min}(D)^2}=\frac{2\norm{E}\cdot\max_{i\in [m]}\norm{\bx_{i|d-k}}^2}{\min{i\in [m]}\norm{\bx_{i|d-k}}^4}\leq \frac{1}{2}<1$ and the fact that $1\leq\frac{\norm{D}}{\lambda_{\min}(D)}$), we get that
\begin{equation}\label{eq:bwdk25}
\left|f_{m}(\bv)-\frac{1}{m}\norm{[D^{-1/2}(\mathbf{1}-Z_{|k}\bv)]_+}^2\right|~\leq~ \frac{2\norm{E}\cdot\norm{D}}{\lambda_{\min}(D)^2}\cdot \frac{1}{m}\norm{[D^{-1/2}(\mathbf{1}-Z_{|k}\bv)]_+}^2~.
\end{equation}
Plugging this back into \eqref{eq:optprob0}, and plugging in $\frac{1}{m}\norm{[D^{-1/2}(\mathbf{1}-Z_{|k}\bv)]_+}^2=\frac{1}{m}\sum_{i=1}^{m}\frac{[1-y_i\bx_{i|k}^\top \bv]_+^2}{\norm{\bx_{i|d-k}}^2}$ (which holds by definition of $D,Z_{|k}$), we get that
\begin{equation}\label{eq:bwdk1}
\hat{\bw}_{|k}~=~ \arg\min_{\bv} \frac{\norm{\bv}^2}{m}+\left(1+\epsilon_{\bv}\right)\cdot\frac{1}{m}\sum_{i=1}^{m}\frac{[1-y_i\bx_{i|k}^\top \bv]_+^2}{\norm{\bx_{i|d-k}}^2}~,
\end{equation}
where $\epsilon_{\bv}\in \reals$ satisfies
\[
\sup_{\bv\in\reals^k}|\epsilon_{\bv}|~\leq~ \frac{2\norm{E}\cdot\norm{D}}{\lambda_{\min}(D)^2}=\epsilon_0~,
\]
proving the first bound in the theorem. 

To get the second bound, note that since $\hat{\bw}_{|k}$ minimizes the expression in \eqref{eq:bwdk1}, which for $\bv=0$ equals $\frac{1+\epsilon_{\bv}}{m}\sum_{i=1}^{m}\frac{1}{\norm{\bx_{i|d-k}}^2}\leq \frac{1+\epsilon_0}{\min_{i}\norm{\bx_{i|d-k}}^2}=\frac{1+\epsilon_0}{\lambda_{\min}(D)}$, it must hold that
\[
\frac{\norm{\hat{\bw}_{|k}}^2}{m}+(1+\epsilon_{\hat{\bw}_{|k}})\cdot\frac{1}{m}\sum_{i=1}^{m}\frac{[1-y_i\bx_{i|k}^\top \hat{\bw}_{|k}]_+^2}{\norm{\bx_{i|d-k}}^2}~\leq~\frac{1+\epsilon_0}{\lambda_{\min}(D)}~,
\]
and since $1+\epsilon_{\hat{\bw}_{|k}}\geq 1-\epsilon_0>0$, it follows that 
\begin{equation}\label{eq:bwdk3}
\frac{1}{m}\sum_{i=1}^{m}\frac{[1-y_i\bx_{i|k}^\top \hat{\bw}_{|k}]_+^2}{\norm{\bx_{i|d-k}}^2}~\leq~ \frac{1+\epsilon_0}{(1-\epsilon_0)\lambda_{\min}(D)}~.
\end{equation}
Next, recall from \eqref{eq:optprob0} and the fact that $(\hat{\bw}_{|k},\hat{\bw}_{|d-k})$ jointly optimize \eqref{eq:optprob00} that
\[
f_{m}(\hat{\bw}_{|k})~=~\frac{\norm{\hat{\bw}_{|d-k}}^2}{m}~.
\]
Combining with \eqref{eq:bwdk25} (with $\bv=\hat{\bw}_{|k}$), it follows that 
\[
\frac{\norm{\hat{\bw}_{|d-k}}^2}{m}\leq
\left(1+\frac{2\norm{E}\cdot\norm{D}}{\lambda_{\min}(D)^2}\right)\cdot\frac{1}{m}\norm{[D^{-1/2}\left(\mathbf{1}-Z_{|k}\hat{\bw}_{|k}\right)]_+}^2=
\left(1+\epsilon_0\right)\cdot\frac{1}{m}\sum_{i=1}^{m}\frac{[1-y_i\bx_{i|k}^\top \hat{\bw}_{|k}]_+^2}{\norm{\bx_{i|d-k}}^2}~,
\]
where we recall that $\epsilon_0=\frac{2\norm{E}\cdot\norm{D}}{\lambda_{\min}(D)^2}$.
Combining this with \eqref{eq:bwdk3}, we get
\[
\frac{\norm{\hat{\bw}_{|d-k}}^2}{m}~\leq~ \frac{(1+\epsilon_0)^2}{(1-\epsilon_0)\lambda_{\min}(D)}~,
\]
which is less than $5/\lambda_{\min}(D)$ (since $\epsilon_0\in [0,\frac{1}{2}]$). Multiplying both sides by $m$, and plugging in $\lambda_{\min}(D)=\min_{i\in[m]}\norm{\bx_{i|d-k}}^2$, results in the second bound in the theorem.

\subsection{Proof of \thmref{thm:classification}}

Assumptions \ref{assump:boundedx} and \ref{assump:E} imply that with probability approaching $1$ as $d$ increases, $\hat{\bw}_d$ exists and the parameter $\epsilon_0$ from \thmref{thm:classificationdet} is arbitrarily small. Under that event,  \thmref{thm:classificationdet} applies, and $\norm{\hat{\bw}_{d|d-k}}^2\leq c_0\cdot m_d$ for some $c_0>0$ independent of $d$. Therefore, with probability approaching $1$,
\begin{align*}
	\E_d\left[\left(\bx^\top\hat{\bw}_d-\bx_{|k}^\top\hat{\bw}_{d|k}\right)^2\right]
	&=~
	\E_d\left[\left(\bx_{|d-k}^\top\hat{\bw}_{d|d-k}\right)^2\right]~=~ \hat{\bw}_{d|d-k}^\top \E_d[\bx_{d-k}\bx_{|d-k}^\top]\hat{\bw}_{d|d-k}\\
	&\leq ~ \norm{\hat{\bw}_{d|d-k}}^2\cdot\norm{\E_d[\bx_{d-k}\bx_{|d-k}^\top]}~\leq~ c_0\cdot m_d\cdot \norm{\E_d[\bx_{d-k}\bx_{|d-k}^\top]}~,
\end{align*}
which by assumption \ref{assump:Sigma} converges to $0$. All this implies that $\E_d\left[\left(\bx^\top\hat{\bw}_d-\bx_{|k}^\top\hat{\bw}_{d|k}\right)^2\right]\stackrel{P}{\longrightarrow} 0$.

We now turn to analyze $\hat{\bw}_{d|k}$, which by \thmref{thm:classificationdet} and assumptions \ref{assump:boundedx},\ref{assump:E} satisfies 
\begin{equation}\label{eq:wdkdef}
\hat{\bw}_{d|k}~=~\arg\min_{\bv\in\reals^k}~(1+\epsilon_{d,\bv})\cdot \hat{g}_d(\bv)+\frac{\norm{\bv}^2}{m_d}
\end{equation}
with probability approaching $1$, where $\hat{g}_d(\bv):= \hat{\E}_d\left[\frac{[1-y\bx_{|k}^\top \bv]_+^2}{\norm{\bx_{|d-k}}^2}\right]=\frac{1}{m_d}\sum_{i=1}^{m_d}\frac{[1-y_i\bx_{i|k}^\top \bv]_+^2}{\norm{\bx_{i|d-k}}^2}$ and  $\sup_{\bv}|\epsilon_{d,\bv}|\leq \epsilon_d$ for some random variable $\epsilon_d$ satisfying $m_d\epsilon_d \stackrel{P}{\longrightarrow} 0$ as $d\rightarrow \infty$.

First, recalling that  $g_d(\bv)=\E_d\left[\frac{[1-y\bx_{|k}^\top \bv]_+^2}{\norm{\bx_{|d-k}}^2}\right]$, we have by assumption \ref{assump:lln} and the law of large numbers that $\hat{g}_d(\bv)-g_d(\bv)\stackrel{P}{\rightarrow} 0$ for any fixed $\bv$. Moreover, by assumption \ref{assump:boundedx}, both $g_d(\cdot)$ and $\hat{g}_d(\cdot)$ are $c_{\Vcal}$-Lipschitz on any fixed compact set $\Vcal$ in $\reals^k$ (where $c_{\Vcal}$ depends on $\Vcal$ but not on $d$, and with probability approaching $1$ for $\hat{g}_{d}$). Therefore, by a standard covering number argument, it follows that 
\[
	\sup_{\bv\in\Vcal}|\hat{g}_d(\bv)-g_d(\bv)|\stackrel{P}{\rightarrow} 0
\]
for any compact set $\Vcal\subset\reals^k$.

We now wish to argue that with probability approaching $1$, $\hat{\bw}_{d|k}$ lies in some compact set $\{\bv\in\reals^k:\norm{\bv}\leq c''\}$ where $c''$ is a constant independent of $d$. Applying the displayed equation above, this would imply that 
\begin{equation}\label{eq:wdkconverge}
\hat{g}_d(\hat{\bw}_{d|k})-g_d(\hat{\bw}_{d|k})\stackrel{P}{\rightarrow} 0~.
\end{equation}
To justify this, let us compare $\hat{\bw}_{d|k}$ to $\hat{\bv}_d$, which we define as the minimum-norm minimizer  of\footnote{A minimizer of $\hat{g}_d(\cdot)$ always exists, since it is convex piecewise-quadratic with finitely many pieces. The minimum-norm minimizer is unique, since if there were two minimizers of equal minimal norm, their average would also be a minimizer by convexity of $\hat{g}_d(\cdot)$, and with a smaller norm which is a contradiction.} $\hat{g}_d(\cdot)$. By \eqref{eq:wdkdef}, we have with probability approaching $1$ for any large enough $d$ that
\begin{align*}
	(1-\epsilon_d)\hat{g}_d(\hat{\bv}_d)+\frac{\norm{\hat{\bw}_{d|k}}^2}{m_d}~&\leq~ (1+\epsilon_{d,\hat{\bw}_{d|k}})\hat{g}_d(\hat{\bw}_{d|k})+\frac{\norm{\hat{\bw}_{d|k}}^2}{m_d}\\
	&\leq~
	(1+\epsilon_{d,\hat{\bv}_d})\hat{g}_d(\hat{\bv}_d)+\frac{\norm{\hat{\bv}_{d}}^2}{m_d}
	~\leq~ (1+\epsilon_d)\hat{g}_d(\hat{\bv}_d)+\frac{\norm{\hat{\bv}_{d}}^2}{m_d}~.
\end{align*}
Multiplying both sides by $m_d$ and switching sides, it follows that
\[
\norm{\hat{\bw}_{d|k}}^2~\leq~ 2m_d \epsilon_d\cdot g_d(\hat{\bv}_d)+\norm{\hat{\bv}_{d}}^2~\leq~ 2c^2m_d\epsilon_d + \norm{\hat{\bv}_d}^2~,
\]
where the last transition follows from $\hat{\bv}_d$ being a minimizer of $\hat{g}_d(\cdot)$, hence $\hat{g}_d(\hat{\bv}_d)\leq \hat{g}_d(\mathbf{0})=\hat{\E}_d[\frac{1}{\norm{\bx_{|d-k}}^2}]\leq c^2$ by assumption  \ref{assump:boundedx}. Recalling that $m_d\epsilon_d  \stackrel{P}{\longrightarrow} 0$, it follows in particular $\norm{\hat{\bw}_{d|k}}^2\leq 1+\norm{\hat{\bv}_d}^2$ with probability approaching $1$. By assumption \ref{assump:boundedmin} in the theorem, $\norm{\hat{\bv}_d}\leq c'$ for some constant $c'$ with probability approaching $1$. Overall, we get that with probability approaching $1$, $\hat{\bw}_{d|k}$ lies in some compact set $\{\bv\in\reals^k:\norm{\bv}\leq c''\}$ with $c''$ independent of $d$, hence justifying \eqref{eq:wdkconverge} as discussed earlier. 

Next, let us fix some reference vector $\bv_0\in\reals^k$. By assumption \ref{assump:classsimple} and the law of large numbers, $\hat{g}_d(\bv_0)-g_d(\bv_0)\stackrel{P}{\rightarrow} 0$. Combined with \eqref{eq:wdkconverge}, it follows that
\[
\delta_d := \left(g_d(\hat{\bw}_{d|k})-g_d(\bv_0)\right)-\left(\hat{g}_d(\hat{\bw}_{d|k})-\hat{g}_d(\bv_0)\right)~\stackrel{P}{\longrightarrow}~0
\]
Thus, we have the following:
\begin{align*}
	g_d(\hat{\bw}_{d|k})-g_d(\bv_0)~&=~ \hat{g}_d(\hat{\bw}_{d|k})-\hat{g}_d(\bv_0)+\delta_d\\
	&=~
	(1+\epsilon_{d,\hat{\bw}_{d|k}})\cdot \hat{g}_d(\hat{\bw}_{d|k})-\hat{g}_d(\bv_0)+\delta_d-\epsilon_{d,\hat{\bw}_{d|k}}\cdot \hat{g}_d(\hat{\bw}_{d|k})\\
	&\leq~
	(1+\epsilon_{d,\hat{\bw}_{d|k}})\cdot \hat{g}_d(\hat{\bw}_{d|k})+\frac{\norm{\hat{\bw}_{d|k}}^2}{m_d}-\hat{g}_d(\bv_0)+\delta_d-\epsilon_{d,\hat{\bw}_{d|k}}\cdot \hat{g}_d(\hat{\bw}_{d|k})\\
	&\stackrel{\eqref{eq:wdkdef}}{\leq}~	(1+\epsilon_{d,\bv_0})\cdot\hat{g}_d(\bv_0)+\frac{\norm{\bv_0}^2}{m_d}-\hat{g}_d(\bv_0)+\delta_d-\epsilon_{d,\hat{\bw}_{d|k}}\cdot \hat{g}_d(\hat{\bw}_{d|k})\\
	&=~\epsilon_{d,\bv_0}\cdot\hat{g}_d(\bv_0)+\frac{\norm{\bv_0}^2}{m_d}+\delta_d-\epsilon_{d,\hat{\bw}_{d|k}}\cdot \hat{g}_d(\hat{\bw}_{d|k})~.
\end{align*}
We now argue that all the terms in the bound above converge (deterministically or in probability) to $0$: As to the first term, we know that $\sup_{\bv}|\epsilon_{d,\bv}|\stackrel{P}{\rightarrow} 0$, and $\hat{g}_d(\bv_0)$ is at most some fixed value independent of $d$ with probability approaching $1$, by assumption \ref{assump:boundedx} and definition of $\hat{g}_d(\cdot)$. As to the second term, it converges to $0$ since $m_d\rightarrow \infty$. The third term converges in probability to $0$ as discussed above. As to the last term, we know that $\sup_{\bv}|\epsilon_{d,\bv}|\stackrel{P}{\rightarrow} 0$, and $\hat{g}_d(\hat{\bw}_{d|k})$ is bounded by some constant independent of $d$ with probability approaching $1$ (since $\hat{\bw}_{d_k}$ lies in some fixed compact set with probability approaching $1$ as discussed earlier, and the values of $\hat{g}_d(\cdot)$ are bounded independent of $d$ on any fixed compact set with probability approaching $1$, by assumption \ref{assump:boundedx}).

The displayed equation above and the following discussion implies that for any $\beta>0$, $\Pr(g_d(\hat{\bw}_{d|k})-g_d(\bv_0)\geq \beta)\rightarrow 0$. This holds for any fixed $\bv_0$. Combined with the assumption $\inf_{\bv}\limsup_d (g_d(\bv)-\inf_{\bu}g_d(\bu))\leq 0$ (which implies that $\limsup_d (g_d(\bv_0)-\inf_{\bv}g_d(\bv))$ can be made arbitrarily small by choosing $\bv_0$ appropriately), it follows that $\Pr(g_d(\hat{\bw}_{d|k})-\inf_{\bv}g_d(\bv)\geq \beta)\rightarrow 0$ for all $\beta>0$. But since $g_d(\hat{\bw}_{d|k})-\inf_{\bv}g_d(\bv)$ is necessarily non-negative, we get that  $g_d(\hat{\bw}_{d|k})-\inf_{\bv}g_d(\bv)\stackrel{P}{\rightarrow} 0$.

\subsection{Proof of \thmref{thm:classificationbenign}}

Since the labels are flipped with some probability $p>0$, we have $\inf_d \inf_{\bw\in\reals^d}R_d(\bw)>0$. Thus, it remains to prove that under the condition stated in the theorem, $\Pr_{(\bx,y)\sim \Dcal_d}(y\bx^\top \hat{\bw}_d)$ converges in probability to $p$. 

As discussed before the theorem, $L_p(\cdot)$ has a unique minimizer $\bw^*_p$. Therefore, by \thmref{thm:classification}, $\hat{\bw}_{d|k}\stackrel{P}{\rightarrow}\bw^*_p$. Since we assume $\Pr_{(\bx_{|k},y)\sim\Dcal_{\text{clean}}}(y\bx_{|k}^\top\bw_p^*\leq 0)=0$, we argue that
\begin{equation}\label{eq:pras0}
	\Pr_{(\bx_{|k},y)\sim\Dcal_{\text{clean}}}(y\bx_{|k}^\top\hat{\bw}_{d|k}\leq 0)~\stackrel{P}{\longrightarrow}~0~.
\end{equation}
We note that formally proving this requires some care, as $\Pr(y\bx_{|k}^\top \bw\leq 0)$ is not necessarily continuous in $\bw$ (otherwise \eqref{eq:pras0} would follow immediately by continuity). 
To show \eqref{eq:pras0} formally, define for all $\gamma>0$ the set $\Ucal_{\gamma}:=\{\bu\in\reals^k:\bu^\top\bw_p^*>\gamma,\norm{\bu}\leq \frac{1}{\gamma}\}$. Clearly, since $\Pr_{(\bx_{|k},y)\sim\Dcal_{\text{clean}}}(y\bx_{|k}^\top\bw_p^*\leq 0)=0$, we have
\begin{equation}\label{eq:zgamma}
\Pr_{(\bx_{|k},y)\sim\Dcal_{\text{clean}}}(y\bx_{|k}\in \Ucal_{\gamma})~\stackrel{\gamma\rightarrow 0}{\longrightarrow} 1~.
\end{equation}
Moreover, since $\hat{\bw}_{d|k}\stackrel{P}{\rightarrow}\bw^*_p$, it holds that $y\bx_{|k}^\top\hat{\bw}_{d|k}\stackrel{P}{\rightarrow}y\bx_{|k}^\top\hat\bw^*_p$ simultaneously for all vectors $y\bx_{|k}$ of some bounded norm. Therefore, for any fixed $\gamma$, 
\begin{align}
	\Pr_{(\bx_{|k},y)\sim\Dcal_{\text{clean}}}&(y\bx_{|k}^\top\hat{\bw}_{d|k}\leq 0~|~y\bx_{|k}\in \Ucal_{\gamma})\notag\\
	&=~
	\Pr_{(\bx_{|k},y)\sim\Dcal_{\text{clean}}}\left(y\bx_{|k}^\top\hat{\bw}_{d|k}\leq 0~|~y\bx_{|k}^\top\bw^*_p>\gamma,\norm{y\bx_{|k}}\leq \frac{1}{\gamma}\right)~\stackrel{P}{\longrightarrow}~0~.\label{eq:pras}
\end{align}
Recalling that for any two events $E,A$ over some probability space, 
\[
\Pr(A)~=~\Pr(A|E)\cdot\Pr(E)+\Pr(A|\neg E)\cdot\Pr(\neg E)~\leq~ \Pr(A|E)+\Pr(\neg E)~,
\]
it follows that for any fixed $\gamma$,
\[
\Pr_{(\bx_{|k},y)\sim\Dcal_{\text{clean}}}\left(y\bx_{|k}^\top\hat{\bw}_{d|k}\leq 0\right)~\leq~\Pr_{(\bx_{|k},y)\sim\Dcal_{\text{clean}}}(y\bx_{|k}^\top\hat{\bw}_{d|k}\leq 0~|~y\bx_{|k}\in \Ucal_{\gamma})+
\Pr(y\bx_{|k}\notin\Ucal_{\gamma}).
\]
Combined with \eqref{eq:zgamma} and \eqref{eq:pras}, it follows that by picking $\gamma$ sufficiently small, we can make\\  $\Pr_{(\bx_{|k},y)\sim\Dcal_{\text{clean}}}\left(y\bx_{|k}^\top\hat{\bw}_{d|k}\leq 0\right)$ asymptotically smaller than any positive number with arbitrarily high probability, from which \eqref{eq:pras0} follows.

From \eqref{eq:pras0} and the definition of $\Dcal_{\text{clean}}$, it follows that 
\[
\Pr_{(\bx,y)\sim\Dcal_d}(y\bx_{|k}^\top\hat{\bw}_{d|k}\leq 0)\stackrel{P}{\rightarrow}p~.
\]
By \thmref{thm:classification}, we also have that
\begin{equation}\label{eq:emark}
	\E_{(\bx,y)\sim \Dcal_d}\left[(\bx^\top \hat{\bw}_d-\bx_{|k}^\top \hat{\bw}_{d|k})^2\right]~\stackrel{P}{\longrightarrow}~0~.
\end{equation} 
Combining the last two observations, we can use similar arguments as above to prove that
\begin{equation}\label{eq:emarkend}
\Pr_{(\bx,y)\sim\Dcal_d}(y\bx^\top\hat{\bw}_{d}\leq 0)~\stackrel{P}{\longrightarrow}~p~,
\end{equation}
which as discussed earlier implies the theorem statement. Formally, let $\tilde{\Dcal}_{d}$ refer to $\Dcal_{d}$, where $y$ is distributed according to the ``clean'' distribution $\Dcal_{\text{clean}}$. Also, let $\Ucal^d_{\gamma}=\{\bu\in\reals^d:\bu_{|k}^\top \bw^*_p>\gamma,\norm{\bu}\leq \frac{1}{\gamma}\}$. Similar to before, we have
\[
\Pr_{(\bx,y)\sim\tilde{\Dcal}_{d}}(y\bx\in \Ucal^d_{\gamma})~\stackrel{\gamma\rightarrow 0}{\longrightarrow} 1~.
\]
By applying Markov's inequality on \eqref{eq:emark}, it follows that for any $\gamma>0$, the measure of points $y\bx\in \Ucal^d_{\gamma}$ such that $|y\bx^\top\hat{\bw}_d-y\bx^\top_{|k}\hat{\bw}_{d|k}|>\frac{\gamma}{2}$ goes to $0$ in probability. For all other points in $\Ucal^d_{\gamma}$, we have $y\bx_{|k}^\top\hat{\bw}_{d|k}> \gamma~\Rightarrow~y\bx^\top \hat{\bw}_d> \frac{\gamma}{2}$. Recalling that $y\bx_{|k}^\top\hat{\bw}_{d|k}\stackrel{P}{\rightarrow} y\bx_{|k}^\top\bw^*_{p}$ (which is $> \gamma$) uniformly for all $y\bx\in \Ucal^d_{\gamma}$, we get that
\[
\Pr_{(\bx,y)\sim\tilde{\Dcal}_d}(y\bx^\top\hat{\bw}_{d}\leq 0~|~y\bx\in \Ucal^d_{\gamma})~\stackrel{P}{\longrightarrow} 0~.
\]
Combining the two displayed equation above, and using the same arguments made in the context of \eqref{eq:zgamma} and \eqref{eq:pras}, it follows that by choosing $\gamma$ small enough, the probability $\Pr_{(\bx,y)\sim\tilde{\Dcal}_d}(y\bx^\top\hat{\bw}_{d}\leq 0)$ can be made asymptotically smaller than any positive numbers with arbitrarily high probability, from which it follows that
\[
\Pr_{(\bx,y)\sim\tilde{\Dcal}_d}(y\bx^\top\hat{\bw}_{d}\leq 0)~\stackrel{P}{\longrightarrow}~0~.
\]
Switching from $\tilde{\Dcal}_d$ to $\Dcal_d$ (which involves flipping $y$ randomly with probability $p$), \eqref{eq:emarkend} follows.

\subsection{Proof of \thmref{thm:universalbo}}

Recall that a function $f:\reals^k\rightarrow \reals$ is $\lambda$-strongly convex, if for any $\bu,\bv\in \reals^k$ and $\alpha\in [0,1]$,
\begin{equation}\label{eq:strongconv}
	f(\alpha \bu+(1-\alpha)\bv)~\leq~ \alpha\cdot f(\bu)+(1-\alpha)f(\bv)-\alpha(1-\alpha)\cdot\frac{\lambda}{2}\norm{\bu-\bv}^2~.
\end{equation}
Note that any $\lambda$-strongly convex function is also $\lambda'$ strongly convex for any $\lambda'\in [0,\lambda]$. Also, it is well-known that any convex function is $0$-strongly convex, that if $f$ is $\lambda$-strongly convex, then $c\cdot f$ is $c\cdot \lambda$-strongly convex, and that a sum of a $\lambda$-strongly convex function and a $\lambda'$-strongly convex function is $(\lambda+\lambda')$-strongly convex. Moreover, if $f$ is $\lambda$-strongly convex, it always has a finite unique minimizer $\bw^*$, and $f(\bw)-f(\bw^*)\geq \frac{\lambda}{2}\norm{\bw-\bw^*}^2$ for any $\bw$.

We start with the following auxiliary lemma, which implies that $L_p(\cdot)$ is strongly convex under Assumption \ref{assump:classsimple} (since positive definiteness of $\E[\bx_{|k}\bx_{|k}^\top]$ implies positive definiteness of $\E[z\cdot \bx_{|k}\bx_{|k}^\top]$ for $z$ bounded in a positive interval):

\begin{lemma}\label{lem:strongconv}
	If $\E[z\cdot\bx_{|k}\bx_{|k}^\top]$ is positive definite (with minimal eigenvalue $\lambda_{\min}>0$), then for any $p
	\in (0,\frac{1}{2}]$, $L_p(\cdot)$ (as defined in \eqref{eq:lp} is $2p\lambda_{\min}$-strongly convex.
\end{lemma}
\begin{proof}
	We have $L_p(\bw)=\E[z\cdot\ell_p(y\bx_{|k}^\top \bw)]$, where $\ell_p(\beta)=(1-p)[1-\beta]_+^2+p[1+\beta]_+^2$. We first argue that $\ell_p$ is $2p$-strongly convex: Indeed, it can be easily verified that $\ell_{p}(\beta)=p\beta^2+(1-2p)[1-\beta]_+^2+h(\beta)$, where $h(\beta)$ is a convex function that equals $-2\beta+1$ on $(-\infty,-1]$, $\beta^2+2$ on $[-1,+1]$, and $2\beta+1$ on $[1,\infty)$. Therefore, $\ell_p$ is the sum of the $2p$-strongly convex function $p\beta^2$ and convex functions, hence is $2p$-strongly convex itself. 
	
	Continuing, note that by the theorem's assumptions, 
	\[
	\E[z\cdot(\bx_{|k}^\top\bu-\bx_{|k}^\top \bv)^2]~=~ (\bu-\bv)^\top \E[z\cdot\bx_{|k}\bx_{|k}^\top]\cdot (\bu-\bv)\geq \lambda_{\min} \norm{\bu-\bv}^2~.
	\] 
	Combining this with the strong convexity of $\ell_p$, we get that for any $\bu\bv\in\reals^k$ and $\alpha\in [0,1]$,
	\begin{align*}
		L_p\left(\alpha \bu+(1-\alpha)\bv\right)~&=~ \E\left[z\cdot \ell_p\left(\alpha \bx_{|k}^\top \bu+(1-\alpha)\bx_{|k}^\top \bv\right)\right]\\
		&\leq~
		\E\left[z\alpha \ell_p(\bx_{|k}^\top \bu)+z(1-\alpha)\ell_p(\bx_{|k}^\top \bv)-p\alpha(1-\alpha)z(\bx_{|k}^\top \bu-\bx_{|k}^\top \bv)^2\right]\\
		&\leq~ \alpha L_p(\bu)+(1-\alpha)L_p(\bv)-p\alpha(1-\alpha)\lambda_{\min}\norm{\bu-\bv}^2
	\end{align*}
	Which by \eqref{eq:strongconv}, implies that $L_p$ is $2p\lambda_{\min}$-strongly convex. 
\end{proof}

We now continue with the proof of the theorem. Recall that
\[
L_p(\bw)=\E\left[(1-p)\cdot z\cdot[1-y\bx_{|k}^\top \bw]_+^2+p\cdot z\cdot[1+y\bx_{|k}^\top\bw]_+^2\right]~,
\]
and in particular, $L_0(\bw)=\E[z\cdot[1-y\bx_{|k}^\top \bw]_+^2]$, which by the linear separability assumption achieves a minimal value of $0$ at some $\bw^*$. For any $p\in [0,\frac{1}{2})$, let 
\[
\tilde{L}_p(\bw)~=~\E[z\cdot[1-y\bx_{|k}^\top\bw]_+^2]+\frac{p}{1-2p}\cdot g(\bw)~~~\text{where}~~~g(\bw):=\E[z\cdot[1-y\bx_{|k}^\top\bw]_+^2+z\cdot[1+y\bx_{|k}^\top\bw]_+^2]~.
\]
It is easy to check that $(1-2p)\cdot \tilde{L}_p(\bw) = L_p(\bw)$ for all $\bw$, hence a minimizer $\bw^*_p$ of $L_p(\cdot)$ is also a minimizer of $\tilde{L}_p(\cdot)$, and $\bw^*$ is a minimizer of $\tilde{L}_0(\cdot)$. Moreover, by \lemref{lem:strongconv}, $\tilde{L}_p(\bw)$ is $\frac{2p\lambda_{\min}}{1-2p}$-strongly convex (which would also imply that its minimizer $\bw^*_p$ always exists and is unique). 

Next, we argue that $\bw^*_p$ is continuous as a function of $p$ in $(0,\frac{1}{2})$: Otherwise, there is some $p_0\in (0,\frac{1}{2})$ and a sequence of values $p_1,p_2,\ldots$ converging to $p_0$, such that $\bw^*_{p_j}$ remains bounded away from $\bw^*_{p_0}$, say by some minimal distance $\delta>0$. Let us see why that is not possible: By $\frac{2p\lambda_{\min}}{1-2p}$-strong convexity of $\tilde{L}_p(\cdot)$ and the fact that $\bw^*_{p}$ is a minimizer, it would imply
\[
\tilde{L}_p(\bw^*_{p_0})-\tilde{L}_p(\bw^*_p)\geq \frac{p\lambda_{\min}}{(1-2p)}\norm{\bw^*_{p_0}-\bw^*_p}^2\geq \frac{p\lambda_{\min} \delta^2}{1-2p}
\]
if $p=p_j$ for some $j$. Similarly, by $\frac{2p_0\lambda_{\min}}{1-2p_0}$-strong convexity of $\tilde{L}_{p_0}(\cdot)$, and the fact that $\bw^*_{p_0}$ is a minimizer, we would have
\[
\tilde{L}_{p_0}(\bw^*_{p})-\tilde{L}_{p_0}(\bw^*_{p_0})\geq \frac{p_0\lambda_{\min}}{1-2p_0}\norm{\bw^*_{p}-\bw^*_{p_0}}^2\geq \frac{p_0\lambda_{\min} \delta^2}{1-2p_0}
\]
for any $p=p_j$. 
Summing the last two displayed equations for any $p=p_j$, it follows that $(\tilde{L}_{p_j}(\bw^*_{p_0})-\tilde{L}_{p_0}(\bw^*_{p_0}))+(\tilde{L}_{p_0}(\bw^*_{p_j})-\tilde{L}_{p_j}(\bw^*_{p_j}))$ is bounded away from $0$ as $j\rightarrow \infty$, but this contradicts the fact that $\tilde{L}_{p_j}(\bw^*_{p_0})- \tilde{L}_{p_0}(\bw^*_{p_0})\stackrel{j\rightarrow \infty}{\longrightarrow}0$ and $\tilde{L}_{p_0}(\bw^*_{p_j})- \tilde{L}_{p_j}(\bw^*_{p_j})\stackrel{j\rightarrow \infty}{\longrightarrow} 0$. 

Now, since $\bw^*_p$ is a continuous function of $p$ in $(0,\frac{1}{2})$, it must have a limit point $\hat{\bw}$ as $p\rightarrow 0$.
Since $\lim_{p\rightarrow 0}\tilde{L}_{p}(\bw^*) = \tilde{L}_0(\bw^*)$ and $\tilde{L}_0(\bw^*)\leq \tilde{L}_0(\bw^*_{p})\leq \tilde{L}_{p}(\bw^*_{p})\leq \tilde{L}_{p}(\bw^*)$, we must have $\lim_{p\rightarrow 0}\tilde{L}_{p}(\bw^*_{p})=\tilde{L}_0(\bw^*)$. But since $\bw^*_{p}\stackrel{p\rightarrow 0}{\longrightarrow} \hat{\bw}$ and $\tilde{L}_{p}$ is Lipschitz in any fixed neighborhood of $\hat{\bw}$ (with a uniform upper bound on the Lipschitz constant), it follows that $\tilde{L}_{p}(\hat{\bw})\stackrel{p\rightarrow 0}{\longrightarrow}\tilde{L}_0(\bw^*)$. Recalling that $\tilde{L}_0(\bw^*)=L_0(\bw^*)=0$, it follows that
\[
\lim_{p\rightarrow 0}\tilde{L}_{p}(\hat{\bw})~=~0~.
\]
Combined with the fact that $\tilde{L}_{p}(\hat{\bw})~\geq~ \E[z\cdot[1-y\bx_{|k}^\top \hat{\bw}]_+^2]~\geq~\E[\frac{z}{4}\mathbf{1}_{y\bx_{|k}^\top \hat{\bw}<\frac{1}{2}}]~\geq~\frac{l}{4}\Pr(y\bx_{|k}^\top \hat{\bw}<\frac{1}{2})$ regardless of $p$, it follows that
\[
\Pr\left(y\bx_{|k}^\top \hat{\bw}<\frac{1}{2}\right)~=~0~.
\]
Now, let $B$ be such that $\Pr(\norm{y\bx_{|k}}\leq B)=1$ (such a $B$ exists by assumption). Note that since $\bw^*_{p}\stackrel{p\rightarrow 0}{\longrightarrow} \hat{\bw}$, then for any $p>0$ sufficiently small, we must have $\norm{\bw^*_{p}-\hat{\bw}}\leq \frac{1}{4B}$. For any $y,\bx_{|k}$ such that $\norm{y\bx_{|k}}\leq B$, the event $y\bx_{|k}^\top \bw^*_{p}< \frac{1}{4}$ implies
\[
y\bx_{|k}^\top \hat{\bw}~=~ y\bx_{|k}^\top \bw^*_{p}+y\bx_{|k}^\top (\hat{\bw}-\bw^*_{p})~<~
\frac{1}{4}+\norm{y\bx_{|k}}\norm{\hat{\bw}-\bw^*_{p}}~\leq~ \frac{1}{4}+\frac{B}{4B}~=~\frac{1}{2}.
\]
But since we showed that $y\bx_{|k}^\top \hat{\bw}<\frac{1}{2}$ occurs with probability $0$, it follows that the event $y\bx_{|k}^\top \bw^*_{p}< \frac{1}{4}$ also occurs with probability $0$, namely
\[
\Pr\left(y\bx_{|k}^\top \bw^*_{p}<\frac{1}{4}\right)~=~0~.
\]
In particular, we get that for all sufficiently small $p$, the misclassification error probability of $\bw^*_{p}$ is $0$.

\subsection{Proof of \thmref{thm:p12}}
	
	We first argue that for all $p\in (0,\frac{1}{2}]$, $\bw^*_p$ must be in $\text{span}(\bu)$. Otherwise, suppose that $\bw^*_p=\alpha\bu+\br$ for some $\alpha\in \reals$ and non-zero vector $\br$ orthogonal to $\bu$. Then we argue that $\alpha\bu-\br$ (which is distinct from $\bw^*_p$) must also be a minimizer of $L_p$, because conditioned on any $y,z$, the distribution of
	\[
	\bx_{|k}^\top \bw^*_p = \bx_{|k}^\top (\alpha\bu+\br) = \alpha \bx_{|k}^\top \bu+\br^\top (I-\bu\bu^\top)\bx_{|k}
	\]
	is the same as
	\[
	\bx_{|k}^\top (\alpha\bu-\br) = \alpha \bx_{|k}^\top \bu-\br^\top (I-\bu\bu^\top)\bx_{|k}~,
	\]
	and the the value of $L_p(\cdot)$ conditioned on any $y,z$ depends just on these quantities. But since $L_p(\cdot)$ is strongly convex, its minimizer must be unique, which is a contradiction. 
	
	Next, let $\bw^*\in \reals^k$ such that $L_0(\bw^*)=\E[z\cdot [1-y\bx_{|k}^\top\bw^*]_+^2]=0$ (such a vector exists by the linear separability assumption). We argue that we can assume $\bw^*\in \text{span}(\bu)$ without loss of generality: If not, and it equals $\alpha\bu+\br$ with $\br\neq 0$ orthogonal to $\bu$, then by the same arguments as above, $\alpha\bu-\br$ also minimizes $L_0(\cdot)$. But $L_0(\cdot)$ is convex, so the average of the two points (which is $\alpha\bu$) is a minimizer of $L_0(\cdot)$, and we can take $\bw^*$ to be that minimizer.
	
	Finally, we argue that if we write $\bw^*_p$ as $\alpha_p\bu$, and $\bw^*$ as $\alpha\bu$, then the sign of $\alpha_p$ and $\alpha$ must be the same. Indeed, suppose without loss of generality that $\alpha>0$ (otherwise, flip $\bu$ to $-\bu$, and note that $\alpha$ cannot be zero, since then $\bw^*=0$ and it cannot possibly satisfy the theorem assumptions). Since $L_0(\bw^*)=\E[z\cdot [1-y\bx_{|k}^\top\alpha \bu]_+^2]=0$ and $z\geq l>0$ with probability $1$, it follows that $zy\bx_{|k}^\top \bu>0$ with probability $1$. Therefore,
	\[
	\frac{d}{d\beta} L_p(\beta \bu)~|_{\beta=0} ~=~ 
	-2(1-2p)\E[zy\bx_{|k}^\top\bu]<0
	\]
	for any $p\in (0,\frac{1}{2})$, 
	which by convexity of $\beta\mapsto L_p(\beta\bu)$ implies that the (unique) minimizer $\bw^*_p=\alpha_p \bu$ of $L_p(\cdot)$ must satisfy $\alpha_p>0$. 
	Overall, we have
	\[
	\Pr(y\bx_{|k}^\top \bw^*_p\leq 0)~=~\Pr(\alpha_p y\bx_{|k}^\top \bu\leq 0)~=~\Pr(\alpha y\bx_{|k}^\top \bu\leq 0)~=~ \Pr(y\bx_{|k}^\top \bw^*\leq 0)~=~0~.
	\]
	as required.

\section{Minimizers of the Squared Hinge Loss Can Lead to Large Misclassification Error}\label{app:hingebad}

Fix some distribution $\Dcal$ over examples $(\bx,y)\in \reals^d\times \{-1,+1\}$. If the distribution is linearly separable on the first $k$ coordinates, it is easy to see that a minimizer of the expectation of the weighted squared hinge loss, $\E\left[\frac{[1-y\bx_{|k}^\top \bw]_+^2}{\norm{\bx_{|d-k}}}\right]$ will also minimize the expected misclassification error (probability that $y\bx_{|k}^\top\bw\leq 0$), since it will return a point $\bw$ such that $1-y\bx_{|k}^\top\bw\leq 0$ with probability $1$ (and such a point exists by the linear separability assumption). However, this can badly break down when there isn't linear separability. Concretely, suppose that $\norm{\bx_{|d-k}}=1$ with probability $1$, and that we introduce label noise, so that the sign of $y$ is flipped with some probability $p$. In this case, the expected loss can be written as
\[
L_p(\bw)~=~\E_{(\bx_{|k},y)}\left[(1-p)\cdot[1-y\bx_{|k}^\top \bw]_+^2+p\cdot[1+y\bx_{|k}^\top \bw]_+^2\right]~,
\]
where the expectation is with respect to the ``clean'' labels. In this case, the minimizer of the above might have an expected misclassification error of $1/2$ (even if $p$ is arbitrarily small). To see this, it is enough to produce some finite linearly-separable dataset, such that $50\%$ of the points will be misclassified by the minimizer of $L_p(\cdot)$ (and then random label flipping will keep the error rate at $50\%$). The existence of such a dataset was essentially shown for a more general setting in \citet{long2010random}, and below we instantiate their analysis for our setting with more explicit guarantees:

\begin{proposition}\label{prop:classbad}
	For any $p\in (0,\frac{1}{12})$, there exists a dataset $\{\bx_i,y_i\}_{i=1}^{4}\subseteq \reals^2\times \{-1,+1\}$, where $\max_i \norm{\bx}_i\leq 1$, such that:
	\begin{itemize}
		\item There exists a unit vector $\bw^*$ for which $\min_i y_i \bx_i^\top \bw^*\geq p$
		\item If $\hat{\bw}$ is a minimizer of
		$L_p(\bw)~=~\frac{1}{4}\sum_{i=1}^{4}\left((1-p)\cdot[1-y_i\bx_i^\top \bw]_+^2+p\cdot[1+y_i\bx_i^\top\bw]_+^2\right)
		$,
		then $\hat{\bw}$ misclassifies two of the four points. 
	\end{itemize} 
\end{proposition}
\begin{proof}
	
	Let $y_1=y_2=y_3=y_4=1$, and
	\[
	\bx_1=\bx_2=\left(\begin{matrix}p\\-p\end{matrix}\right)~,~\bx_3=\left(\begin{matrix}1\\0\end{matrix}\right)~,~\bx_4=\left(\begin{matrix}p\\5p\end{matrix}\right)~.
	\]
	It is easily verified that $\bw^*=(p,0)$ satisfies $\min_i y_i \bx_i^\top \bw^*\geq p$. Also, by \lemref{lem:strongconv}, it is easily verified that $L_p(\cdot)$ is strongly convex. Therefore, the minimizer is unique, and we claim that for any small enough $p>0$, it equals
	\[
	\hat{\bw}=\left(\frac{5-16p}{3+8p}~,~\frac{1-p}{3p(3+8p)}\right)~.
	\]
	In that case, for the two points $\bx_1=\bx_2=(p,-p)$, 
	\[
	\bx_1^\top\hat{\bw}~=~\bx_2^\top \hat{\bw}~=~-\frac{1-16p+48p^2}{9+24p}~,
	\]
	which is negative for any small enough $p\in (0,\frac{1}{12})$, hence two of the four points are misclassified. 
	To verify that $\hat{\bw}$ above is indeed the minimizer, let $\ell_p(z):=(1-p)[1-z]_+^2+p[1+z]_+^2$ (so that $L_p(\bw)=\E_{(\bx,y)}[\ell_p(y\bx^\top\bw)]$), and note that 
	\begin{equation}\label{eq:objgrad}
		4\cdot\nabla L_p(\hat{\bw})=4\cdot \nabla L_p(\hat{w}_1,\hat{w}_2)~=~2p\ell'_p(p(\hat{w}_1-\hat{w}_2))\left(\begin{matrix}1\\-1\end{matrix}\right)+\ell_p'(w_1)\left(\begin{matrix}1\\0\end{matrix}\right)+p\ell'_p(p(\hat{w}_1+5\hat{w}_2))\left(\begin{matrix}1\\5\end{matrix}\right),
	\end{equation}
	where 
	\[
	\ell'_p(z)~=~-2(1-p)[1-z]_++2p[1+z]_+~.
	\]
	A tedious but routine calculation shows that for any $p\in (0,\frac{1}{12})$, it holds that $p(\hat{w}_1-\hat{w}_2)\in [-1,0)$, $p(\hat{w}_1+5 \hat{w}_2)\in [0,1]$, and $\hat{w}_1>1$. Plugging in the corresponding expressions for $\ell_p'(z)$ into \eqref{eq:objgrad}, we get the $\mathbf{0}$ vector. Hence, $\nabla L_p(\hat{\bw})$, and since $L_p(\cdot)$ is convex, it follows that $\hat{\bw}$ is indeed its minimizer. 
\end{proof}

\end{document}